%% file: neurips_submit.tex
\documentclass{article}
    
\PassOptionsToPackage{numbers}{natbib}
\usepackage[final]{neurips_2020}

\usepackage[utf8]{inputenc} 
\usepackage[T1]{fontenc}    
\usepackage{url}            
\usepackage{amsfonts}       
\usepackage{nicefrac}       
\usepackage{microtype}      

\usepackage{caption}
\usepackage{subcaption}
\usepackage{wrapfig}
\usepackage{paralist}
\usepackage{amsmath}
\usepackage{mathrsfs}
\usepackage{color}
\usepackage{xcolor}
\usepackage{multirow}
\usepackage{colortbl}
\usepackage{threeparttable}
\usepackage{graphicx}
\usepackage{enumitem}
\usepackage{adjustbox}
\usepackage{hyperref}
\usepackage{microtype}
\usepackage{comment}
\usepackage{graphicx}
\usepackage{booktabs} 
\usepackage{amsthm}
\usepackage{paralist}
\usepackage{algorithm}
\usepackage[noend]{algorithmic}
\usepackage{comment}
\usepackage{subcaption}
\usepackage{wrapfig}
\usepackage{tablefootnote}
\usepackage{titlesec}

\newcolumntype{b}{>{\columncolor[gray]{0.8}} c}
\newtheorem{theorem}{Theorem}
\newtheorem{remark}{Remark}
\newtheorem{lemma}[theorem]{Lemma}

\input{math_command.tex}

\title{Automatic Perturbation Analysis for\\Scalable Certified Robustness and Beyond}
\author{
Kaidi Xu\textsuperscript{1,*},\enskip Zhouxing Shi\textsuperscript{2,*},\enskip Huan Zhang\textsuperscript{3,*},\enskip Yihan Wang\textsuperscript{2} \\
\rule[10pt]{0pt}{0pt}
\textbf{Kai-Wei Chang\textsuperscript{3}, Minlie Huang\textsuperscript{4}, Bhavya Kailkhura\textsuperscript{5}, Xue Lin\textsuperscript{1}, Cho-Jui Hsieh\textsuperscript{3}} \\
\rule[15pt]{0pt}{0pt}
{\normalsize \textsuperscript{1}Northeastern University \enskip \textsuperscript{2}Tsinghua University \enskip \textsuperscript{3}UCLA}\\
{\normalsize \textsuperscript{4}DCST, THUAI, SKLits, BNRist, Tsinghua University}
{\, \normalsize \textsuperscript{5}Lawrence Livermore National Laboratory}\\
{\tt\small \{xu.kaid, xue.lin\}@northeastern.edu, zhouxingshichn@gmail.com, huan@huan-zhang.com},\\
{\tt \small wangyihan617@gmail.com, kw@kwchang.net, aihuang@tsinghua.edu.cn},\\ 
{\tt\small kailkhura1@llnl.gov, chohsieh@cs.ucla.edu}\\
\rule[15pt]{0pt}{0pt}
\textsuperscript{*}{\it Kaidi Xu, Zhouxing Shi and Huan Zhang contributed equally}
}

\date{}

\begin{document}

\maketitle

\setlength{\abovedisplayskip}{3.0pt plus 0.0pt minus 2.0pt}
\setlength{\belowdisplayskip}{3.0pt plus 0.0pt minus 2.0pt}
\setlength{\abovedisplayshortskip}{0.0pt plus 1.0pt}
\setlength{\belowdisplayshortskip}{2.0pt plus 0.0pt minus 1.0pt}


\input{abstract}

\input{introduction}

\setlength{\textfloatsep}{7pt}
\setlength{\floatsep}{5pt}
\setlength{\intextsep}{5pt}

\input{related}

\input{algorithm}

\input{experiments}



\section*{Broader Impact}
In this paper, we develop an automatic framework to enable perturbation analysis on any neural network structures. Our framework can be used in a wide variety of tasks ranging from robustness verification to certified defense, and potentially many more applications requiring a provable perturbation analysis. It can also play an important building block for several safety-critical ML applications, such as transportation, engineering, and healthcare, etc. We expect that our framework will significantly improve the robustness and reliability of real-world ML systems with theoretical guarantees.

An important product of this paper is an open-source LiRPA library with over 10,000 lines of code, which provides automatic and differentiable perturbation analysis. This library can tremendously facilitate the use of LiRPA for the research community as well as industrial applications, such as verifiable plant control~\citep{wong2020neural}. Our library of LiRPA on general computational graphs can also inspire further improved implementations on automatic outer bounds calculations with provable guarantees.
  
Although our focus on this paper has been on exploring known perturbations and providing guarantees in such clairvoyant scenarios, in real-world an adversary (or nature) may not adhere to our assumptions. Thus, we may additionally want to understand implication of these unknown scenarios on the system performance. This is a relatively unexplored area in robust machine learning, and we encourage researchers to understand and mitigate the risks arising from unknown perturbations in these contexts.

\section*{Acknowledgments and Disclosure of Funding}
This work was performed under the auspices of the U.S. Department of Energy by Lawrence Livermore National Laboratory under Contract DE-AC52-07NA27344 and was partly supported by the National Science Foundation CNS-1932351,
NSFC key project No. 61936010, NSFC regular project No. 61876096, 
NSF IIS-1901527, NSF IIS-2008173, and ARL-0011469453.

\bibliographystyle{preprint}
\bibliography{bib}

\newpage

\appendix

\input{appendix.tex}

\end{document}

%% file: math_command.tex
\usepackage{amsmath,amsfonts,bm}

\def\1{\bm{1}}

\def\eps{{\epsilon}}

\def\rva{{\mathbf{a}}}
\def\rvb{{\mathbf{b}}}
\def\rvc{{\mathbf{c}}}
\def\rvd{{\mathbf{d}}}
\def\rve{{\mathbf{e}}}

\def\rvg{{\mathbf{g}}}
\def\rvh{{\mathbf{h}}}

\def\rvw{{\mathbf{w}}}
\def\rvx{{\mathbf{x}}}

\def\rmA{{\mathbf{A}}}

\def\rmE{{\mathbf{E}}}

\def\rmG{{\mathbf{G}}}

\def\rmI{{\mathbf{I}}}

\def\rmV{{\mathbf{V}}}
\def\rmW{{\mathbf{W}}}
\def\rmX{{\mathbf{X}}}

\def\vzero{{\bm{0}}}
\def\vone{{\bm{1}}}

\def\mLambda{{\bm{\Lambda}}}
\def\mDelta{{\bm{\Delta}}}

\DeclareMathAlphabet{\mathsfit}{\encodingdefault}{\sfdefault}{m}{sl}
\SetMathAlphabet{\mathsfit}{bold}{\encodingdefault}{\sfdefault}{bx}{n}

\def\sS{{\mathbb{S}}}

\newcommand{\R}{\mathbb{R}}



%% file: abstract.tex
\begin{abstract}
Linear relaxation based perturbation analysis (LiRPA) for neural networks, which computes provable linear bounds of output neurons given a certain amount of input perturbation, has become a core component in robustness verification and certified defense. The majority of LiRPA-based methods focus on simple feed-forward networks and need particular manual derivations and implementations when extended to other architectures. In this paper, we develop an {\it automatic} framework to enable perturbation analysis on any neural network structures, by generalizing existing LiRPA algorithms such as CROWN to operate on general computational graphs. The flexibility, differentiability and ease of use of our framework allow us to obtain state-of-the-art results on LiRPA based certified defense on fairly complicated networks like DenseNet, ResNeXt and Transformer that are not supported by prior works. Our framework also enables {\it loss fusion}, a technique that significantly reduces the computational complexity of LiRPA for certified defense. For the first time, we demonstrate LiRPA based certified defense on Tiny ImageNet and Downscaled ImageNet where previous approaches cannot scale to due to the relatively large number of classes. Our work also yields an open-source library for the community to apply LiRPA to areas beyond certified defense without much LiRPA expertise, e.g., we create a neural network with a provably flat optimization landscape by applying LiRPA to network parameters. Our open source library is available at \textcolor{blue}{\url{https://github.com/KaidiXu/auto_LiRPA}}.
\end{abstract}

%% file: introduction.tex
\section{Introduction}

Bounding the range of a neural network outputs given a certain amount of input perturbation has become an important theme for neural network verification and certified adversarial  defense~\citep{kolter2017provable,mirman2018differentiable,wang2018mixtrain,zhang2019towards}. 
However, computing the exact bounds for output neurons is usually intractable~\cite{katz2017reluplex}. 
Recent research studies have developed perturbation analysis bounds that are sound,
computationally feasible, and relatively tight~\citep{kolter2017provable,zhang2018efficient,singh2019abstract,weng2018towards,singh2018fast,wang2018efficient}. For a neural network function $f(\rvx) \in \R$, to study its behaviour at $\rvx_0$ with bounded perturbation $\bm{\delta}$ such that $\rvx=\rvx_0+\delta\in \sS$ (e.g., $\sS$ is a $\ell_p$ norm ball around $\rvx_0$), these works provide two linear functions 
$\underline{f}(\rvx) := \underline{\rva}^\top \rvx + \underline{\rvb}$ and $\overline{f}(\rvx) := \overline{\rva}^\top \rvx + \overline{\rvb}$ 
that are guaranteed lower and upper bounds respectively for output neurons w.r.t. the input under perturbation:
$\underline{f}(\rvx) \leq f(\rvx) \leq \overline{f}(\rvx) \ (\forall \rvx\in\sS)$. 
We refer to this line of work as a {\bf Li}near {\bf R}elaxation based {\bf P}erturbation {\bf A}nalysis ({\bf LiRPA}).
Beyond its usage in neural network verification and certified defense, LiRPA is capable to serve as a general toolbox to understand the behavior of deep neural networks (DNNs) within a predefined input region, and has been demonstrated useful for interpretation and explanation of DNNs~\cite{ko2019popqorn,shi2020robustness}.

To compute LiRPA bounds, the first step is to obtain linear relaxations of any non-linear units~\citep{zhang2018efficient,salman2019convex} (e.g., activation functions) in a network. Then, these relaxations need to be ``glued'' together according to the network structure to obtain the final bounds. Early developments of LiRPA focused on feed-forward networks, and it has been extended to a few more complicated network structures for real-world applications. For example, \citet{wong2018scaling} implemented LiRPA for convolutional ResNet on computer vision tasks; \citet{zugner2019certifiable} extended~\cite{kolter2017provable} to graph convolutional networks; \citet{ko2019popqorn} and  \citet{shi2020robustness} extended CROWN~\cite{zhang2018efficient} to recurrent neural networks and Transformers respectively. Unfortunately, each of these works extends LiRPA with an ad-hoc implementation that only works for a specific network architecture.
This is similar to the ``pre-automatic differentiation'' era where researchers have to implement gradient computation by themselves for their designed network structure. Since LiRPA is significantly more complicated than backpropagation, non-experts in neural network verification can find it challenging to understand and use LiRPA for their purpose.

Our paper takes a big leap towards making LiRPA a useful tool for general machine learning audience, by generalize existing LiRPA algorithms to general computational graphs. Our framework is a superset of many existing works~\citep{wong2018provable,zhang2018efficient,weng2018towards,ko2019popqorn,shi2020robustness}, and 
our automatic perturbation analysis algorithm is analogous to automatic differentiation.
Our algorithm can compute LiRPA automatically for a given PyTorch model without manual derivation or implementation for the specific network architecture. Importantly, our LiRPA bounds are differentiable which allows efficient training of these bounds.
In addition, our proposed framework enables the following contributions: 
\begin{itemize}[wide,itemsep=0pt]
    \item The flexibility and ease-of-use of our framework allow us to easily obtain state-of-the-art certified defense results for fairly complicated networks, such as DenseNet, ResNeXt and Transfomer that no existing work supports due to tremendous efforts required for manual LiRPA implementation.
    \item We propose {\it loss fusion}, a technique that significantly reduces the computational complexity of LiPRA for certified defense. We demonstrate the first LiPRA-based certified defense training on Tiny ImageNet and Downscaled ImageNet~\citep{chrabaszcz2017downsampled}, with a \emph{two-magnitude improvement} on training efficiency.
    \item Our framework allows flexible perturbation specifications beyond $\ell_p$-balls. For example, we demonstrate a {\it dynamic programming} approach to concretize linear bounds under discrete perturbation of synonym-based word substitution in a sentiment analysis task.
    \item We showcase that LiRPA can be a \emph{valuable tool beyond adversarial robustness}, by demonstrating how to create a neural network with a provably flat optimization landscape and revisit a popular hypothesis on generalization and the flatness of optimization landscape. This is enabled by our unified treatment and automatic derivation of LiRPA bounds for parameter space variables (model weights).
\end{itemize}

%% file: related.tex
\section{Background and Related Work}
Giving certified lower and upper bounds for neural networks under input perturbations is the core problem in robustness verification of neural networks. Early works formulated robustness verification for ReLU networks as satisfiability modulo theory (SMT) and integer linear programming (ILP) problems~\cite{ehlers2017formal,katz2017reluplex,tjeng2017evaluating}, which are hardly feasible even for a MNIST-scale small network. \citet{wong2018provable} proposed to relax the verification problem with linear programming and investigated its dual solution. Many other works have independently discovered similar algorithms~\cite{dvijotham2018dual,mirman2018differentiable,singh2018fast,weng2018towards,zhang2018efficient,singh2019abstract,wang2018efficient} in either primal or dual space which we refer to as linear relaxation based perturbation analysis (LiRPA). Recently, \citet{salman2019convex} unified these algorithms under the framework of convex relaxation. Among them, CROWN~\cite{zhang2018efficient} and DeepPoly~\cite{singh2019abstract} achieve the tightest bound for efficient single neuron linear relaxation and are representative algorithms of LiRPA. Several further refinements for the LiRPA bounding process were also proposed recently, including using an optimizer to choose better linear bounds~\citep{dvijotham2018training,lyu2019fastened}, relaxing multiple neurons~\citep{singh2019beyond} or further tighten convex relaxations~\citep{tjandraatmadja2020convex}, but these methods typically involve much higher computational costs. The contribution of our work is to extend LiRPA to its most general form, and allow automatic derivation and computation for general network architectures. Additionally, our framework allows a general purpose perturbation analysis for any nodes in the graph and flexible perturbation specifications, not limiting to perturbations on input nodes or $\ell_p$-ball perturbation specifications. This allows us to use LiRPA as a general tool beyond robustness verification.

The neural network verification problem can also be solved via many other techniques, for example, semidefinite programming~\cite{dvijothamefficient2019,raghunathan2018semidefinite}, bounding local or global Lipschitz constant~\cite{hein2017formal,raghunathan2018semidefinite,zhang2018recurjac}. However, LiRPA based verification methods typically scale much better than alternatives, and they are a keystone for many state-of-the-art certified defense methods.
Certified adversarial defenses typically seek for a guaranteed upper bound on test error, which can be efficiently obtained using LiRPA bounds. By incorporating the bounds into the training process (which requires them to be efficient and differentiable), a network can become certifiably robust~\cite{kolter2017provable,mirman2018differentiable,wang2018mixtrain,gowal2018effectiveness,zhang2019enhancing}.
In addition, while interval bound propagation (IBP)~\cite{mirman2018differentiable,gowal2018effectiveness} that propagates constant bounding intervals can be easily extended to general computational graphs, 
bounds computed by IBP can be very loose and make stable training challenging~\citep{zhang2019towards}. 
Along with these methods, randomization based probabilistic defenses have been proposed~\citep{cohen2019certified,li2018second,lecuyer2019certified,salman2019provably}, but in this work we mostly focus on LiRPA based deterministic certified defense method.

Backpropagation~\cite{rumelhart1986learning} is a classic algorithm to compute the gradients of a complex error function. It can be applied automatically once the forward computation is defined, without manual derivation of gradients. It is essential in most deep learning frameworks, such as TensorFlow~\cite{abadi2016tensorflow} and PyTorch~\cite{NIPS2019_9015}. The backward \emph{bound} propagation in our framework is analogous to backpropagation as our computation is also automatic given the computational graph created by forward propagation, but we aim to automatically derive bounds for output neurons instead of gradients. Our algorithm is significantly more complicated. On the other hand, LiRPA based bounds have been implemented manually in many previous works~\citep{wong2018provable,zhang2018efficient,wang2018mixtrain,maurereran2018}, but they mostly focus on specific types of networks (e.g., feedforward or residual networks) for their empirical study, and do not have the flexibility to generalize to general computational graphs and irregular networks.

%% file: algorithm.tex
\section{Algorithm}

\input{notations}

\subsection{Framework of Perturbation Analysis on General computational Graphs}
\paragraph{Notations} We define a computational graph as a Directed Acyclic Graph (DAG) $\rmG = (\rmV, \rmE)$. 
$\rmV=\{1, 2, \cdots, n\}$ is a set of nodes in $\rmG$. $\rmE$ is a set of node pairs $(i, j)$ which denotes that node $i$ is an input argument of node $j$. For simplicity, we denote the in-degree of node $i$ as $m(i)$, and the set of input nodes for node $i$ as $u(i) = \{u_1(i),\cdots,u_{m(i)}(i)\}$ where $(u_j(i), i) \in \rmE, 1\leq j\leq m(i)$.  
Each node $i$ has a few associated attributes: $H_i(\cdot)$ is the associated computation function, $\rvh_i = H_i(u(i))$ is the vector produced by node $i$. Although $\rvh_i$ can be a tensor in practice, we assume it has been flattened into a vector for simplicity in this paper. Each node $i$ is either an \emph{independent node} with $m(i)\!=\!0$ representing the input nodes of the graph (e.g., network parameters, model inputs), or a \emph{dependent node} representing some computations (e.g., ReLU, MatMul). 
For independent nodes, $H_i$ is an identity function and we denote $\rvh_i\!=\!\rvx_i$.
We let $\rmX$ be the concatenation of all $\rvx_i$, such that the output of each node $i$ can be written as a function of $\rmX$, $\rvh_i\!=\!h_i(\rmX)$, without explicitly referring to $u_{j}(i)$. Without losing generality, we assume that the computational graph has a single output node $o$.
To conduct perturbation analysis, we consider $\rvx_i$ to be arbitrarily taken from an \emph{input space} $\sS_i$. In particular, if $\rvx_i$ is not perturbed, $\sS_i=\{\rvc_i\}$ and $\rvc_i$ is a constant vector.
We denote $\sS$ to be the space of $\rmX$ when each part of $\rmX$, $\rvx_i$, is perturbed within $\sS_i$ respectively.

\paragraph{Linear Relaxation based Perturbation Analysis (LiRPA)} 
Our final goal is to compute provable lower and upper bounds for the value of output node $h_o(\rmX)$, i.e., lower bound $\underline{\rvh}_o$ and upper bound $\overline{\rvh}_o$ (element-wise), when $\rmX$ is perturbed within $\sS$:
$    \underline{\rvh}_o \leq h_o(\rmX) \leq \overline{\rvh}_o, \enskip \forall \rmX\in \sS$.
In LiRPA, we find tight lower and upper bounds by first computing linear bounds w.r.t. $\rmX$:
\begin{equation}
    \underline{\rmW}_o \rmX + \underline{\rvb}_o \leq h_o(\rmX) \leq \overline{\rmW}_o \rmX + \overline{\rvb}_o \quad \forall \rmX \in \sS,
    \label{eq:linear_bounds}
\end{equation}
where $h_o(\rmX)$ is bounded by linear functions of $\rmX$ with parameters $\underline{\rmW}_o, \underline{\rvb}_o, \overline{\rmW}_o, \overline{\rvb}_o$. We generalize existing LiRPA approaches into two categories: \emph{forward mode} perturbation analysis and \emph{backward mode} perturbation analysis. Both methods aim to obtain bounds~\eqref{eq:linear_bounds} in different manners:

\begin{itemize}[wide,itemsep=0pt]
\item \textbf{Forward mode}: forward mode LiRPA propagates the linear bounds of each node w.r.t. all the independent nodes, i.e., linear bounds w.r.t. $\rmX$, to its successor nodes in a forward manner, until reaching the \emph{output node} $o$. 
\item \textbf{Backward mode}: backward mode LiRPA propagates the linear bounds of \emph{output node} $o$ w.r.t. \emph{dependent nodes} to further predecessor nodes in a backward manner, until reaching all the \emph{independent nodes}.

\end{itemize}
We describe these two different modes in details below.

\paragraph{Forward Mode LiRPA on General Computation Graphs} For each node $i$ on the graph, 
we compute the linear bounds of $h_i(\rmX)$ w.r.t. all the independent nodes:
\[\small
    \underline{\rmW}_i \rmX + \underline{\rvb}_i \leq h_i(\rmX) \leq \overline{\rmW}_i \rmX + \overline{\rvb}_i \quad \forall \rmX \in \sS.
    \label{eq:linear_bounds_forward}
\]
We start from independent nodes. For an independent node $i$, we have $h_i(\rmX) \!=\! \rvx_i$ so we trivially have the bounds $\rmI \rvx_i\! \leq\! h_i(\rmX) \!\leq\! \rmI \rvx_i$. For a dependent node $i$,
we have a \emph{forward LiRPA oracle function} $G_i$ which takes $\underline{\rmW}_j$, $\underline{\rvb}_j$, $\overline{\rmW}_j$, $\overline{\rvb}_j$ for every $j \!\in\! u(i)$ as input and produce new linear bounds for node $i$,
assuming all node $j\in u(i)$ have been bounded:
\begin{align}
        (\underline{\rmW}_i, \underline{\rvb}_i, \overline{\rmW}_i, \overline{\rvb}_i) = 
    G_i ( \{ B_{j} | j \in u(i) \}), 
    \text{where }  B_j := (\underline{\rmW}_j, \underline{\rvb}_j, \overline{\rmW}_j, \overline{\rvb}_j).  
    \label{eq:forward_local}
\end{align}
We defer the discussions on oracle function $G_i$ to a later section. Now, we focus on extending this method on a general graph with known oracle functions in Algorithm~\ref{alg:forward}.
\begin{algorithm}[t!]
    \caption{Forward Mode Bound Propagation on General Computational Graphs}
    \begin{algorithmic}
        \FUNCTION{BoundForward($i$)}
            \FOR{$j \in u(i)$}
                \IF{attributes $\underline{\rmW}_j, \underline{\rvb}_j, \overline{\rmW}_j, \overline{\rvb}_j$ of node $j$ are unavailable}\STATE{BoundForward(j)}
                \ENDIF
            \ENDFOR
            \STATE  $(\underline{\rmW}_i, \underline{\rvb}_i, \overline{\rmW}_i, \overline{\rvb}_i) = 
    G_i ( \{ B_{j} | j \in u(i) \})$
        \ENDFUNCTION
    \end{algorithmic}
    \label{alg:forward}
\end{algorithm}
The forward mode perturbation analysis is straightforward to extend to a general computational graph: for each dependent node $i$, we can obtain its bounds by recursively applying~\eqref{eq:forward_local}.
We check every input node $j$ and compute the bounds of node $j$ if they are unavailable.
We then use $G_i$ to obtain the linear bounds of node $i$. The correctness of this procedure is guaranteed by the property of $G_i$: given $B_j$ as inputs, it always produces valid bounds for node $i$. We analyze its complexity in Appendix \ref{apd:complexity}.

\paragraph{Backward Mode LiRPA on General Computation Graphs} 
For each node $i$, we maintain two attributes: $\underline{\rmA}_i$ and $\overline{\rmA}_i$,
representing 
the coefficients in the linear bounds of $h_o(\rmX)$ w.r.t $h_i(\rmX)$:
\begin{equation}\small
    \sum_{i\in\rmV} \underline{\rmA}_i h_i(\rmX) + \underline{\rvd} 
    \leq h_o(\rmX) 
    \leq 
    \sum_{i\in\rmV} \overline{\rmA}_i h_i(\rmX) + \overline{\rvd} \quad \forall\rmX\in\sS,
    \label{eq:backward_sum}
\end{equation}
where $\underline{\rvd},\overline{\rvd}$ are bias terms that are maintained in our algorithm.
Suppose that the output dimension of node $i$ is $s_i$, then the shape of matrices $\underline{\rmA}_i$ and $\overline{\rmA}_i$ is $s_o\!\times\! s_i$. 
Initially, we trivially have 
\begin{equation}\small
\underline{\rmA}_o=\overline{\rmA}_o=\rmI, \enskip
\underline{\rmA}_i=\overline{\rmA}_i=\vzero(i\neq o),\enskip
\underline{\rvd} =\overline{\rvd}=\vzero,
\label{eq:backward_init}
\end{equation}
which makes \eqref{eq:backward_sum} hold true.
When node $i$ is a dependent node, 
we have a \emph{backward LiRPA oracle function} $F_i$ aiming to compute the lower bound of 
$\underline{\rmA}_ih_i(\rmX)$ and the upper bound of $\overline{\rmA}_ih_i(\rmX)$, and represent the bounds with linear functions of its predecessor nodes $u_1(i),u_2(i),\cdots,u_{m(i)}(i)$:
\newcommand{\newA}{{\mLambda}}
\newcommand{\newd}{{\mDelta}}
\begin{align}\small
(\underline{\newA}_{u_1(i)},\overline{\newA}_{u_1(i)},\underline{\newA}_{u_2(i)},\overline{\newA}_{u_2(i)},\cdots,\underline{\newA}_{u_{m(i)}(i)},\overline{\newA}_{u_{m(i)}(i)},\underline{\newd},\overline{\newd})=
	F_i(\underline{\rmA}_i,\overline{\rmA}_i),\nonumber\\
	\text{s.t.} \quad
	\sum\nolimits_{j\in u(i)} \underline{\newA}_j h_j(\rmX) + \underline{\newd} 
	\leq \underline{\rmA}_ih_i(\rmX),
	\enskip
	\overline{\rmA}_ih_i(\rmX)\leq
	\sum\nolimits_{j\in u(i)} \overline{\newA}_j h_{j}(\rmX) + \overline{\newd}.
	\label{eq:backward_local}
\end{align}
We substitute the $h_i(\rmX)$ terms in \eqref{eq:backward_sum} with the new bounds~\eqref{eq:backward_local}, and thereby these terms are backward propagated to the predecessor nodes and replaced by the $h_{j}(\rmX)(j\in u(i))$ related terms in~\eqref{eq:backward_local}. 
In the end, all such terms are propagated to the independent nodes and $h_o(\rmX)$ will be bounded by linear functions of  independent nodes only, where \eqref{eq:backward_sum} becomes equivalent to \eqref{eq:linear_bounds}.

{
\begin{algorithm}[t!]
    \caption{Backward Mode Bound Propagation on a General Computational Graph}
    \begin{algorithmic}\FUNCTION{BoundBackward($o$)}
            \STATE Create BFS queue $Q$ and $Q.push(o)$
            \STATE $\underline{\rmA}_o\!\leftarrow\! \rmI, \ \ \overline{\rmA}_o\!\leftarrow\!\rmI,\ \ \underline{\rmA}_i\!\leftarrow\!\vzero, \ \ \overline{\rmA}_i\!\leftarrow\!\vzero\ \ (\forall i\neq o),\ \ \underline{\rvd}\!\leftarrow\!\vzero,\ \ \overline{\rvd}\!\leftarrow\!\vzero\ \ $ (Eq. \eqref{eq:backward_init})
            \STATE $\text{GetOutDegree}(o)$ \COMMENT{\textcolor{brown}{$\forall i$ obtain $d_i$, the number of unprocessed output nodes of node $i$ that $o$ depends on.}}
            \WHILE{$Q$ is not empty}
            \STATE $i \leftarrow Q.pop()$ 
            \STATE $(\underline{\newA}_{u_1(i)},\overline{\newA}_{u_1(i)},\underline{\newA}_{u_2(i)},\overline{\newA}_{u_2(i)},\cdots,\underline{\newA}_{u_{m(i)}(i)},\overline{\newA}_{u_{m(i)}(i)},\underline{\newd},\overline{\newd})=
	F_i(\underline{\rmA}_i,\overline{\rmA}_i)$\ \ (Eq. \eqref{eq:backward_local})
            \FOR{$j\in u(i)$}
            {
            \STATE $\underline{\rmA}_{j}+\!\!=\underline{\newA}_j,\ \ \overline{\rmA}_{j}+\!\!=\overline{\newA}_j,\ \ d_j-\!\!=1$
                \IF{$d_j=0$ and node $j$ is a dependent node}
                    \STATE $Q.push(j)$ 
                \ENDIF
            }
            \ENDFOR
            \STATE $\underline{\rvd}+\!\!=\underline{\newd},\ \ \overline{\rvd}+\!\!=\overline{\newd},\ \ \underline{\rmA}_i\!\leftarrow\!\vzero, \ \ \overline{\rmA}_i\!\leftarrow\!\vzero$ \COMMENT{\textcolor{brown}{Clear $\underline{\rmA}_i$ and $\overline{\rmA}_i$ once we propagated through $i$}.}
            \ENDWHILE
            \STATE \textbf{return} $\underline{\rvd}$, $\overline{\rvd}$ \COMMENT{\textcolor{brown}{The algorithm has modified $\underline{\rmA}_i$, $\overline{\rmA}_i$ on the graph.}}
        \ENDFUNCTION
    \end{algorithmic}
    \label{alg:backward}
\end{algorithm}
}

\begin{figure}[t]  
\centering
\vspace{0.2in}
\begin{tabular}{c}
\includegraphics[width=.9\textwidth]{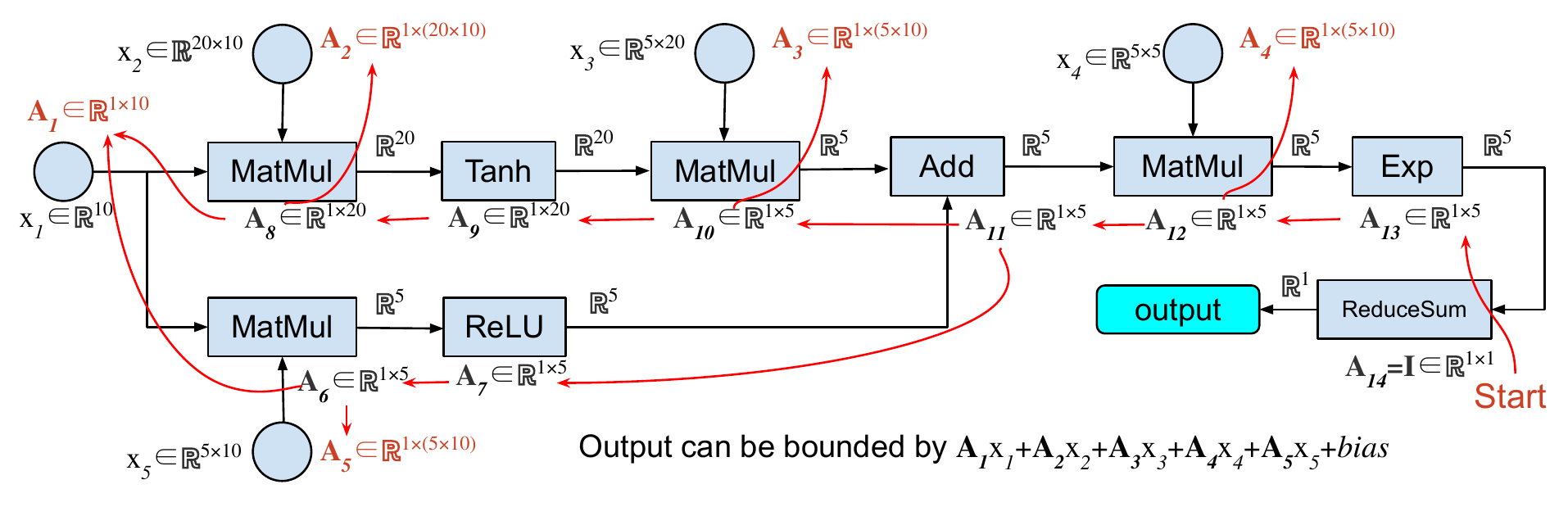}
 \end{tabular}
\caption{Illustration of the backward mode perturbation analysis.
Node $1\sim 5$ are independent nodes and the others are dependent nodes.
Red arrows represent the flow of $\rmA$ matrices including both $\underline{\rmA}$ and $\overline{\rmA}$ that are propagated from the final output node (node 14) to previous nodes.
Finally, only independent nodes retain non-zero $\rmA$ matrices (highlighted in red), and these matrices represent linear bounds w.r.t. independent nodes.}
\label{fig:framework}
\end{figure}

\begin{figure}[t]  
\centering
\begin{tabular}{c}
 \includegraphics[width=.9\textwidth]{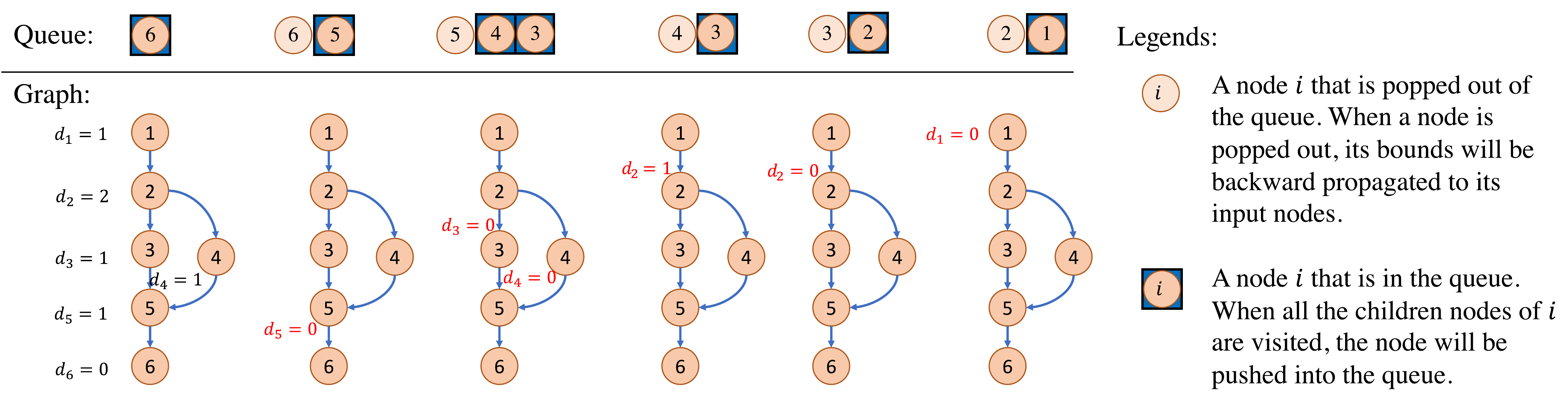}
 \end{tabular}
\caption {Flowchart of the BFS in Algorithm~\ref{alg:backward}. In this example, node 6 is the final output node and $d_i$ is the number of unprocessed output nodes of node $i$ that node $6$ depends on.}
\label{fig:backward}
\end{figure}

We present the full algorithm 
in Algorithm~\ref{alg:backward}.
We let $d_i$ denote the number of unprocessed output nodes of node $i$ that node $o$ depends on, which is initially obtained by a ``GetOutDegree'' function detailed in Appendix~\ref{apd:getdegree}.
We use a BFS for propagating the linear bounds, starting from node $o$ as \eqref{eq:backward_init}.
For each node $i$ picked from the head of the queue, we 
backward propagate $h_i(\rmX)$ using~\eqref{eq:backward_local}.
We update the bound parameters and
decrease all $d_j(j\in u(i))$ by one.
If $d_j\!=\!0$ becomes true for a dependent node $j$, all its related successor nodes have been processed and we push node $j$ to the queue. We repeat this process until the queue is empty.
Figure~\ref{fig:framework} illustrates the flow of backward propagating the bound parameters on an example computational graph, and 
Figure~\ref{fig:backward} illustrates the BFS algorithm.
We show its soundness in Theorem~\ref{theorem:backward} and its proof is given in Appendix \ref{apd:theorm1_proof}.

\begin{theorem}[Soundness of backward mode LiRPA]
When Algorithm~\ref{alg:backward} terminates, we have 
\begin{equation*}\small
    \sum_{i\in\rmV} \underline{\rmA}_i h_i(\rmX) + \underline{\rvd} 
    \leq h_o(\rmX) 
    \leq 
    \sum_{i\in\rmV} \overline{\rmA}_i h_i(\rmX) + \overline{\rvd} \quad \forall\rmX\in\sS,
\end{equation*}
where $\underline{\rmA}_i$, $\overline{\rmA}_i$ are guaranteed to be $\bm{0}$ for all dependent nodes, and thus we obtain provable linear upper and lower bounds of node $o$ w.r.t. all independent nodes.
\label{theorem:backward}
\end{theorem}

\paragraph{Oracle Functions}

Oracle functions $F_i$ and $G_i$ are defined for each type of operations.\footnote{Note that the oracle functions of some operations also require $\underline{\rvh}_j,\overline{\rvh}_j(j\in u(i))$ for linear relaxation, although we do not explicitly mention them in the algorithm description for simplicity.}
 Previous works~\cite{kolter2017provable,zhang2018efficient,salman2019convex,shi2020robustness} have covered many common operations such as affine transformations, activation functions, matrix multiplication, etc. 
Since the major focus of this paper is on handling general computational graph structures, rather than deriving bounds for these elementary operations, we left the detailed form of these oracle functions in Appendix~\ref{apd:functions}.

Some oracle functions depend on certain graph attributes.
For example, $F_i$ of node $i$ with a nonlinear operation typically requires $\underline{\rvh}_j$, $\overline{\rvh}_j$ for all $j \in u(i)$ (typically referred to as ``pre-activation bounds'' in previous works). We can obtain $\underline{\rvh}_j$, $\overline{\rvh}_j$ by assuming node $j$ as the output node and apply
Algorithm~\ref{alg:backward}, then concretize the linear bounds as will be discussed in Sec~\ref{sec:spec}. However, this can be very expensive because Algorithm~\ref{alg:backward} needs to be applied for every node $j$ wherever $\underline{\rvh}_j$ or $\overline{\rvh}_j$ is required, rather than just the output node. A typically more efficient approach is to obtain $\underline{\rvh}_j$ or $\overline{\rvh}_j$ for all dependent nodes except $o$ using a cheaper method and then apply backward mode LiRPA for node $o$ only. 
This leads to two variants of hybrid approaches, \emph{Forward+Backward} and \emph{IBP+Backward}, where $\underline{\rvh}_j$ and $\overline{\rvh}_j$ are produced by Foward LiRPA or IBP, respectively. 
For certified training, \emph{IBP+Backward}  (generalized from \citet{zhang2019towards}) is the best for efficiency. We discuss the time complexity of these methods in Appendix~\ref{apd:complexity}.

\subsection{General Perturbation Specifications and Bound Concretization}\label{sec:spec}

Once the linear bounds are obtained as~\eqref{eq:linear_bounds}, concrete bounds $\underline{\rvh}_o$ and $\overline{\rvh}_o$ can be found by solving the following optimization problems (this step is referred to as the ``concretization'' of linear bounds):
\begin{equation*}
    \underline{\rvh}_o = \min_{\rmX\in \sS} \underline{\rmW}_o \rmX + \underline{\rvb}_o,\ \ 
    \overline{\rvh}_o = \max_{\rmX\in \sS} \overline{\rmW}_o \rmX + \overline{\rvb}_o.
\end{equation*}
We show two examples: classic $\ell_p$-ball perturbations, and synonym-based word substitution in language tasks.

\paragraph{$\ell_p$-ball Perturbations}
In this setting, assuming that $\rmX_0$ is the clean input, the input space is defined by  
 $   \sS = \{ \rmX \mid \parallel \rmX - \rmX_0 \parallel_p \leq \epsilon \}$,
which means that the actual input $\rmX$ is perturbed within an $\ell_p$-ball centered at $\rmX_0$ with a radius of $\epsilon$.
Linear bounds can be concretized as~\citet{zhang2018efficient}:
\begin{align*}
    &\underline{\rvh}_o
    =
     -\epsilon \!\parallel\! \underline{\rmW}_o \!\parallel\!_q \!+\! \underline{\rmW}_o \rmX_0 \!+\! \underline{\rvb}_o,\quad 
    \overline{\rvh}_o
    =
    \epsilon \!\parallel\! \overline{\rmW}_o \!\parallel\!_q + \overline{\rmW}_o \rmX_0 \!+\! \overline{\rvb}_o,
    \quad 1/p+1/q=1,
\end{align*}
where $\parallel\cdot\parallel_q$ denotes taking $\ell_q$-norm for each row in the matrix and the result makes up a vector.

\paragraph{Synonym-based Word Substitution}

Beyond $\ell_p$-ball perturbations, we show an example of a perturbation specification defined by synonym-based word substitution in language tasks.
Let the clean input to the model be a sequence of words $w_1, w_2, \cdots, w_l$ mapped to embeddings $e(w_1), e(w_2), \cdots, e(w_l)$.
Following a common adversarial perturbation setting in NLP~\cite{huang2019achieving,jia2019certified}, we allow at most $\delta$ words to be replaced and each word $w_i$ can be replaced by words within its pre-defined substitution set $\sS(w_i)$. $\sS(w_i)$ is constructed from the synonyms of $w_i$ and validated with a language model.
We denote each actual input word as $\hat{w}_i\in \{ w_i \} \cup \sS(w_i)$, and we show that the linear bounds of node $k$ can be concretized with dynamic programming (DP) in Theorem~\ref{theorem:dp} as proved in Appendix~\ref{apd:theorem2_proof}.

\begin{theorem}
Let $\underline{\Tilde{\rmW}}_t$ 
be columns in $\underline{\rmW}_o$ 
that correspond to the coefficients of $e(\hat{w}_t)$ in the linear bounds. 
The lower bound of
$
\underline{\rvb}_o + \sum\nolimits_{t=1}^i \underline{\Tilde{\rmW}}_t e(\hat{w}_t),
$
when $j$ words among $\hat{w}_1,  \dots, \hat{w}_i$ have been replaced, denoted as $\underline{\rvg}_{i,j}
$, can be computed by:
\\\resizebox{\linewidth}{!}{
  \begin{minipage}{\linewidth}
\begin{equation*}
    \underline{\rvg}_{i,j} =
    	\min(
    		\underline{\rvg}_{i-1,j} + \underline{\Tilde{\rmW}}_i e(w_i),\ \ 
    \underline{\rvg}_{i-1,j-1} \!+\! \min_{w'}\{\underline{\Tilde{\rmW}}_i e(w') \} 
    	) \ \  (i,j>0)\quad
    	\text{s.t.} \ \  w' \in \sS(w_i),
\end{equation*}
 \end{minipage}
}\\
and $\underline{\rvg}_{i,0} \!=\! \underline{\rvb}_o \!+\! \sum_{t=1}^i \underline{\Tilde{\rmW}}_t e(w_t)$. 
The concrete lower bound is $\min_{j=0}^\delta\underline{\rvg}_{n,j}$.
The upper bound can also be computed similarly by taking the maximum instead of the minimum in the above DP computation.
\label{theorem:dp}
\end{theorem}

\subsection{\emph{Loss Fusion} for Scalable Training of Certifiably Robust Neural Networks}
\label{sec:training}

The optimization problem of robust training can be formulated as minimizing the robust loss:
{\begin{equation}
\min_\theta \sum_{\rmX_0, y} \max_{\rmX\in\sS} L(
f_\theta(\rmX),y), 
\end{equation}}
where $f_\theta(\rmX)$ is the network output at the logit layer, and $y$ is the ground truth.
Let
$g_\theta(\rmX,y)=(\rve_y \vone^\top-\rmI) f_\theta (\rmX)$ be the margin between the ground truth label and all the classes (similarly defined in~\citet{wong2018provable,zhang2019towards}). 
In previous works, the cross-entropy loss is upper bounded by lower bounds on margins, as a consequence of Theorem 2 in~\citet{wong2018provable}: $\max\nolimits_{\rmX\in\sS} {L}(f_\theta (\rmX), y) \leq L(\underline{g}_\theta(\rmX, y),y)$ where $\underline{g}_\theta(\rmX, y) \leq \min\nolimits_{\rmX\in\sS} {g}_\theta(\rmX, y)$. This requires us to first lower bound $g_\theta(\rmX, y)$ using LiRPA. The most efficient LiRPA approach~\citep{zhang2019towards} used IBP+backward to obtain this bound, requiring $O(Kr)$ time where $K$ is the output (logit) layer size (or number of labels), 
and $O(r)$ is the time complexity of a regular  propagation without computing bounds
(see Appendix~\ref{apd:complexity}). This cannot scale to large datasets when $K$ is large (e.g. in Tiny ImageNet $K=200$; in ImageNet $K=1000$).

We propose a new technique, \emph{loss fusion}, which computes an upper bound of $L(f_\theta (\rmX), y)$ directly without $\underline{g}_\theta(\rmX, y)$ as a surrogate. This is possible by treating $L$ as the output node of the computational graph. When $L$ is the cross entropy loss, we have $L(g_\theta(\rmX),y)=\log S(\rmX,y)$, where
$S(\rmX,y)=\sum_{i\leq K} \exp([-g_\theta(\rmX,y)]_i)$. We can thus compute a LiRPA lower bound for $S(\rmX,y)$ directly.
This is a novel method that has not appeared in previous works and it yields two benefits.
First, this reduces the time complexity of upper bounding $L(f_\theta(\rmX),y)$ to $O(r)$, as now the output layer size has been reduced from $K$ to 1. This is the first time in the literature that a tight LiRPA based bound can be computed in the \emph{same asymptotic complexity as forward propagation} and IBP. 
Second, we show that this is not only faster, but also produces tighter bounds in certain cases:

\begin{theorem}
Given same concrete lower and upper bounds of $g_\theta(\rmX,y)$ as $\underline{g}_\theta(\rmX,y)$  and $\overline{g}_\theta(\rmX,y)$ which may be used in linear relaxation, for $S(\rmX,y)\!=\!\sum_{i\leq K} \exp([-g_\theta(\rmX,y)]_i)$, we have
\begin{equation}
\small{
\max_{\rmX\in\sS} L(f_\theta(\rmX), y)
\leq \log\underline{S}(\rmX,y)
\leq L(-\underline{g}_\theta(\rmX,y),y),
}
\end{equation}	
where $L$ is the cross-entropy loss, $\underline{S}(\rmX,y)$ is the lower bound of $S(\rmX,y)$ by backward mode LiRPA.
\label{theorem:loss_fusion}
\end{theorem}
\vspace{-0.1in}
This theorem is proved in Appendix~\ref{apd:theorem3_proof}.
Intuitively, the original approach of propagating $\underline{g}_\theta(\rmX, y)$ through the cross-entropy loss is similar to using IBP for bounding the loss function, but in \emph{loss fusion} we treat the loss function as part of the computational graph and apply LiRPA bounds to it directly; it produces tighter bounds as we can use a tighter relaxation for the nonlinear function $S(\rmX,y)$.

%% file: notations.tex
\begin{table}[tb]
\centering
\caption{Table of Notations}
\adjustbox{max width=\textwidth}{
\begin{tabular}{ll|ll}
\hline
Symbol & Meanings & Symbol & Meanings\\ \hline
$i$, $j$, $k$ & Any node on a computational graph & $\rvx_i$ & Value of an independent node, typically model input or parameters. \\
$o$ & Output node on a computational graph& $\underline{\rvh}_i$, $\overline{\rvh}_i$ & Lower/upper bound of node $i$ respectively  \\
$m(i)$ & In-degree of node $i$ & $\underline{\rmW}_i,\underline{\rvb}_i$, $\overline{\rmW}_i, \overline{\rvb}_i$ & Parameters of linear lower/upper bounds of node $i$ respectively\\ 
$u(i)$ & Set of predecessor nodes (inputs) of node $i$ & $\underline{\rmA}_i$, $\overline{\rmA}_i$ & Linear coefficients of $h_i(\rmX)$ terms in the linear lower/upper bounds of $h_o(\rmX)$ \\
$\sS$ & The space of the perturbed input& $\underline{\rvd}, \overline{\rvd}$ & Bias terms in the linear lower/upper bounds of $h_o(\rmX)$ during bound propagation\\
$\rmX$ & Concatenation of all $\rvx_i$ (assumed flattened) & $h_i(\rmX)$ & Computed value of node $i$ on a computational graph \\
\hline
\end{tabular}
}
\end{table}

%% file: experiments.tex
\section{Experiments}
\label{sec:exp}

\begin{table*}[htb]
\caption{Error rates of different certifiably trained models on CIFAR-10 and Tiny-ImageNet datasets (results on downscaled ImageNet are in Table~\ref{tab:DSImageNet_error_rate}). ``Standard'', `PGD'' and ``verified'' rows report the standard test error, test error under PGD attack, and verified test error, respectively.}
\adjustbox{max width=0.56\textwidth}{
 \begin{minipage}{\textwidth}
\begin{threeparttable}
\centering
\begin{tabular}{c|c|cb|cb|cb|cb|c|c|c}\toprule
\multirow{2}{*}{Dataset} & \multirow{2}{*}{Error} & \multicolumn{2}{c|}{CNN-7+BN} & \multicolumn{2}{c|}{DenseNet} & \multicolumn{2}{c|}{WideResNet} & \multicolumn{2}{c|}{ResNeXt} & \multicolumn{3}{c}{Literature results} \\
 & &IBP & Ours & IBP & Ours &IBP & Ours & IBP & Ours& CROWN-IBP\cite{zhang2019towards} & IBP\cite{zhang2019towards}\tnote{a} & \citet{balunovic2020adversarial} \\\hline 
\multirow{3}{*}{\shortstack{CIFAR-10\\$\epsilon=\frac{8}{255}$}}& Standard & 57.95\% & \textbf{53.71\%} & 57.21\% &56.03\%  & 58.07\% & 53.89\%  & 56.32\% & 53.85\% & 54.02\%  & 58.43\% & 48.3\% \\
& PGD & 67.10\% & \textbf{64.31\%} & 67.75\% & 65.09\% & 67.23\% & 64.42\% & 67.55\%  &64.16\% & 65.42\% & 68.73\% & -\\
& Verified & 69.56\% & \textbf{66.62\%} & 69.59\% &67.57\% &  70.04\% &67.77\% & 70.41\% & 68.25\% & 66.94\% & 70.81\% & 72.5\% \\\hline

\multirow{3}{*}{\shortstack{Tiny-ImageNet\\$\epsilon=\frac{1}{255}$}}& Standard &  78.54\% & 78.42\% & 78.40\% & 77.96\% & 73.54\% & \textbf{72.18\%} & 78.94\% & 78.58\% & \multicolumn{3}{c}{\multirow{3}{*}{\shortstack{None. \citep{gowal2018effectiveness} reported a IBP model trained on \\$64\times64$ downscaled Imagenet dataset with\\ 84.04\% clean error and  93.87\% verified error.}}}  \\
& PGD & 81.05\% & 80.96\% & 80.32\% & 80.52\% & 79.40\% & \textbf{79.48\%} & 80.17\% & 79.80\%  \\
& Verified & 87.96\% & 87.31\% & 86.87\% & 85.44\% & 85.15\% & \textbf{84.14\%} & 87.70\% & 86.95\%  \\

\bottomrule
\end{tabular}
  \begin{tablenotes}
  \item[a] \citet{gowal2018effectiveness} reported better IBP verified error (68.44\%) but this result was found not easily reproducible~\citep{zhang2019towards,balunovic2020adversarial}
  \end{tablenotes}
  \end{threeparttable}
  \end{minipage}
  }

\label{table:vision_error_rate}
\vspace{-2mm}
\end{table*}

\begin{table}[t!]
\centering
\caption{Per-epoch training time and memory usage of the 4 large models on CIFAR-10 with batch size 256, and 3 large models on Tiny ImageNet with batch size 100. ``LF''=loss fusion; ``OOM''=out of memory. Numbers in parentheses are multiples of natural training time or memory usage. With loss fusion, LiRPA based bounds are only 3 to 5 times slower than natural training even on datasets with many labels. Without loss fusion (e.g., in~\citep{zhang2019towards}) LiRPA cannot scale to the TinyImageNet dataset. }
\footnotesize
  \adjustbox{max width=1.\textwidth}{
\begin{tabular}{c|c|cccc|cccc}
\toprule[1pt]
\multirow{2}{*}{Dataset}& \multirow{2}{*}{Training method}& \multicolumn{4}{c|}{Wall clock time (second)}  & \multicolumn{4}{c}{GPU Memory Usage (GB)}            \\
          \cline{3-10}
          && Natural & IBP & LiRPA w/o LF & LiRPA w/ LF &  Natural & IBP & LiRPA w/o LF & LiRPA w/ LF  \\
          \hline
\multirow{4}{*}{CIFAR-10}&CNN-7+BN &11.89&22.23 (1.87$\times$)&56.05 (4.71$\times$)&33.40 (2.81$\times$)&4.42&7.06 (1.60$\times$)&20.52 (4.64$\times$)&10.34 (2.34$\times$)\\
&DenseNet &22.07&54.40 (2.46$\times$)&OOM&90.79 (4.11$\times$)&6.58&16.78 (2.55$\times$)&OOM&27.50 (4.18$\times$)\\
&WideResNet&19.39&43.65 (2.55$\times$)&OOM&74.78 (3.85$\times$)&7.18&13.50 (1.88$\times$)&OOM&21.98 (3.06$\times$)\\
&ResNeXt &14.78&32.44 (2.20$\times$)&132.70 (8.98$\times$)&55.84 (3.78$\times$)&4.74&11.34 (2.39$\times$)&43.68 (9.21$\times$)&18.58 (3.92$\times$)\\
\midrule
\multirow{4}{*}{Tiny-ImageNet}&CNN-7+BN & 56.70 & 112.09 (1.98$\times$) & OOM & 163.29 (2.88$\times$)  & 4.22 & 7.12 (1.69$\times$) & OOM & 10.57 (2.50$\times$) \\
&DenseNet & 135.17 & 318.77 (2.36$\times$) & OOM & 513.96 (3.80$\times$) & 8.55 & 20.55 (2.4$\times$) & OOM & 34.81 (4.07$\times$) \\
&WideResNet& 133.11  & 407.74 (3.06$\times$) & OOM & 635.50 (4.77$\times$) & 10.91  & 24.05 (2.20$\times$) & OOM & 39.08 (3.58$\times$) \\
&ResNeXt&  92.63 &  191.34 (2.07$\times$) & OOM & 337.83 (3.65$\times$) &  4.31 & 7.05  (1.64$\times$) & OOM & 11.66 (2.69$\times$) \\

\bottomrule[1pt]
\end{tabular}
}
\label{table:cost_cifar}
\end{table}

\begin{wraptable}{r}{0.5\textwidth}
\centering
\caption{\small{Certified defense on Downscaled ImageNet dataset. We use WideResNet in this experiment.}}
\label{tab:DSImageNet_error_rate}
\resizebox{0.5\textwidth}{!}{
\begin{tabular}{c|c|ccc}
\toprule
Dataset & Method & Clean&  PGD &Verified \\
 \hline
\multirow{2}{*}{\shortstack{ImageNet ($64\times 64$)\\$\epsilon=\frac{1}{255}$}} & IBP~\citep{gowal2018effectiveness}   & 84.04\% & 90.88\%  & 93.87\% \\ 
& Ours   & \textbf{83.77}\% & \textbf{89.74}\% & \textbf{91.27}\% \\

\bottomrule
\end{tabular}
}
\end{wraptable}

\paragraph{Robust Training of Large-scale Vision Models}
\label{sec:vision_exp}
Our \emph{loss fusion} technique allows us to scale to Tiny-ImageNet~\citep{le2015tiny} and downscaled ImageNet~\citep{chrabaszcz2017downsampled}; to the best of our knowledge, this is the first LiRPA based certified defense on Tiny-ImageNet and downscaled ImageNet with a large number of class labels (200 and 1000, respectively).
Besides, the automatic LiRPA bounds allow us to train certifiably robust models on complicated network architectures (WideResNet~\citep{zagoruyko2016wide}, DenseNet~\citep{huang2017densely} and ResNeXt~\citep{xie2017aggregated}) and achieve state-of-the-art results, where previous works use simpler models~\citep{wong2018scaling,mirman2018differentiable,wang2018mixtrain,zhang2019towards} due to implementation difficulty. We extend CROWN-IBP \citep{zhang2019towards} to the general IBP+backward approach: we use IBP to compute bounds of intermediate nodes of graph and use tight backward mode LiRPA for the bounds of the last layer. Unlike in CROWN-IBP, we apply loss fusion to avoid the time complexity dependency on the number of class labels, and we train a few state-of-the-art classification models (\citep{zhang2019towards} used a simple CNN feedforward network). 
We compare our results to IBP training~\cite{gowal2018effectiveness}. We provide detailed hyperparameters in Appendix~\ref{sec:app_vision_training}. We report results on 
CIFAR-10~\citep{krizhevsky2009learning} with $\ell_\infty$ perturbation $\epsilon\!=\!8/255$ and Tiny-ImageNet with $\epsilon\!=\!1/255$ in Table~\ref{table:vision_error_rate}, and  Downscaled-ImageNet~\cite{chrabaszcz2017downsampled} which has $1,000$ class labels with $\ell_\infty$ perturbation $\epsilon\!=\!1/255$ in Table~\ref{tab:DSImageNet_error_rate}. 
We find that in all settings, our tight LiRPA bounds improve both clean and verified errors compared to IBP. Additionally, we achieve \emph{state-of-the-art verified error} of $66.62\%$ on CIFAR-10 with $\epsilon\!=\!8/255$, better than latest published works~\citep{gowal2018effectiveness,zhang2019towards,balunovic2020adversarial} in certified defense.

In Table~\ref{table:cost_cifar}, we report wall clock time and GPU memory usage for regular training, pure IBP training, LiRPA training on logit layer without loss fusion (same as~\citep{zhang2019towards}) and LiRPA training with loss fusion. We use the same batch size 256 for all settings and conduct the experiments on 4 Nvidia GTX 1080Ti GPUs. 
With loss fusion, LiRPA is efficient and only 3-4 times slower than natural training on both CIFAR-10 and Tiny ImageNet. With loss fusion, we can enable LiRPA at a cost similar to IBP, allowing us to use much tighter bounds and obtain better-verified errors than IBP (Table~\ref{table:vision_error_rate}).
The computational cost is significantly better than~\citep{zhang2019towards} which is up to 10 (number of labels) times slower than natural training on CIFAR-10, and impossible to scale to Tiny ImageNet with 200 labels or downscaled ImageNet with 1000 labels. We also report an additional comparison where we use the largest possible batch size rather than a fixed batch size in each setting in Appendix~\ref{sec:app_vision_training}.

\begin{table}[t!]
  \centering
    \caption{
  Verification and certified defense for LSTM and Transformer based NLP models. $\delta_\text{train}$ and $\delta$ represent the number of perturbed synonym words during training and evaluation. 
  For the most important setting $\delta_\textrm{train}\!=\!6$, we run training with 5 different seeds and report the mean and standard deviation. 
  $\delta_\mathrm{train}\!=\!0$ stands for natural training (no robust objective); $\delta=0$ stands for evaluating clean (standard) test accuracy. 
  ``IBP+Backward (alt.)'' on $\delta_\mathrm{train}\!=\!1$ has an alternative training schedule focusing on the small $\delta$ (see Appendix~\ref{apd:SST_training}). 
   }
  \footnotesize
  \adjustbox{max width=.95\textwidth}{
  \begin{tabular}{c|cc|ccccccc}
    \toprule[1pt]
    \multirow{2}{*}{Model} & 
    \multicolumn{2}{c|}{Training} &\multicolumn{6}{c}{Verified Test Accuracy (\%)}\\
    & Budget & Method & $\delta=0$ & $\delta=1$ & $\delta=2$ & $\delta=3$ & $\delta=4$ & $\delta=5$ & $\delta=6$ \\
	\hline

	\multirow{8}{*}{LSTM} & \multirow{3}{*}{$\delta_\textrm{train}=0$} & IBP & 84.9 & 0.6 & 0.6 & 0.6 & 0.6 & 0.6 & 0.6\\
	& & Forward & 84.9 & 0 & 0 & 0 & 0 & 0 & 0 \\
	& & Forward+Backward & 84.9 & 0 & 0 & 0 & 0 & 0 & 0\\
	
	\cline{2-10}
		
	& \multirow{3}{*}{$\delta_\textrm{train}=1$} & IBP & 81.3 & 78.2 & 78.2 & 78.2 & 78.2 & 78.2 & 78.2\\
	& & IBP+Backward (alt.) & 81.7 & 77.3 & 75.2 & 73.8 & 72.7 & 72.3 & 72.0\\
	& & IBP+Backward & 81.3 & 79.0 & 78.6 & 78.6 & 78.6 & 78.6 & 78.6\\	
	\cline{2-10}
	& \multirow{2}{*}{$\delta_\textrm{train}=6$} & IBP & 79.8$\pm$1.09  & 76.2$\pm$1.67  & 76.2$\pm$1.67  & 76.2$\pm$1.67  & 76.2$\pm$1.67  & 76.2$\pm$1.67  & 76.2$\pm$1.67 \\
	& & IBP+Backward & 79.4$\pm$1.47  & 76.6$\pm$1.42  & 76.6$\pm$1.42  & 76.6$\pm$1.42  & 76.6$\pm$1.42  & 76.6$\pm$1.42  & 76.6$\pm$1.42 \\
	
	\hline	
	
	\multirow{8}{*}{Transformer} & \multirow{3}{*}{$\delta_\textrm{train}=0$} & IBP & 82.0 & 0.6 & 0.6 & 0.6 & 0.6 & 0.6 & 0.6\\
	& & Forward & 82.0 & 60.6 & 47.1 & 40.5 & 36.8 & 35.6 & 35.0\\
	& & Forward+Backward & 82.0 & 65.0 & 51.2 & 44.5 & 41.3 & 39.2 & 38.7\\
	
	\cline{2-10}
	
	& \multirow{3}{*}{$\delta_\textrm{train}=1$} & IBP & 78.7 &  76.9 &  76.9 & 76.9&  76.9&  76.9&  76.9\\
	& & IBP+Backward (alt.) & 79.2 & 77.0 & 75.4 & 75.1 & 74.5 & 74.1 & 73.9\\
	& & IBP+Backward & 78.5 & 77.3 & 77.2 & 77.1 & 77.1 & 77.1 & 77.1\\
	
	\cline{2-10}
	
	& \multirow{2}{*}{$\delta_\textrm{train}=6$} & IBP & 78.4$\pm$0.34 & 76.6$\pm$0.30 & 76.6$\pm$0.30& 76.6$\pm$0.30& 76.6$\pm$0.30& 76.6$\pm$0.30& 76.6$\pm$0.30\\
	& & IBP+Backward & 78.5$\pm$0.08 & \textbf{77.4$\pm$0.21}  & \textbf{77.4$\pm$0.19}  & \textbf{77.4$\pm$0.19}  & \textbf{77.4$\pm$0.20}  & \textbf{77.4$\pm$0.20}  & \textbf{77.4$\pm$0.19}\\
\bottomrule[1pt]
  \end{tabular}
  }
  \label{table:nlp}
\end{table}

\paragraph{Verifying and Training Robust NLP Models} 
Previous works were only able to implement simple algorithms such as IBP on simple (e.g. CNN and LSTM) NLP models~\cite{jia2019certified,huang2019achieving} for certified defense. None of them can handle complicated models like Transformer~\cite{vaswani2017attention} or train with tighter LiRPA bounds.
We show that our algorithm can train certifiably robust models for LSTM and Transfomrer sentiment classifiers on SST-2~\cite{socher2013recursive}. 
We consider synonym-based word substitution with $\delta\!\leq\!6$ (up to 6 word substitutions). We provide more backgrounds and training details in Appendix \ref{apd:SST_training}.
In Table~\ref{table:nlp}, we first verify \emph{normally trained} ($\delta_\textrm{train}\!=\!0$) LSTM and Transformer. 
Unfortunately, most configurations cannot yield a non-trivial verified accuracy (larger than 1\%), except for the case of using the forward mode and forward+backward mode perturbation analysis 
on a Transformer.
We then conduct certified defense with $\delta_\textrm{train}\!=\!\{1,6\}$ using IBP as in~\citep{jia2019certified,huang2019achieving} and our efficient IBP+Backward perturbation analysis. 
Models trained using IBP+Backward outperforms pure IBP (similar to our observations in computer vision tasks), and the verified test accuracy is significantly better than naturally trained models. 
The results demonstrate that our framework allows us to better verify and train complex NLP models using tight LiRPA bounds.

\paragraph{Training Neural Networks with Guaranteed Flatness}
Recently, some researchers~\cite{he2019asymmetric, jastrzebski2018finding,goyal2017accurate, hoffer2017train} have hypothesized that DNNs optimized with stochastic gradient descent (SGD) can find wide and flat local minima which may be associated with good generalization performance. 

\begin{figure}[t!]  
\centering
\hspace{-15mm}
\begin{subfigure}{0.45\linewidth}
   \centering
\begin{tabular}{cc}
 \includegraphics[width=.45\textwidth]{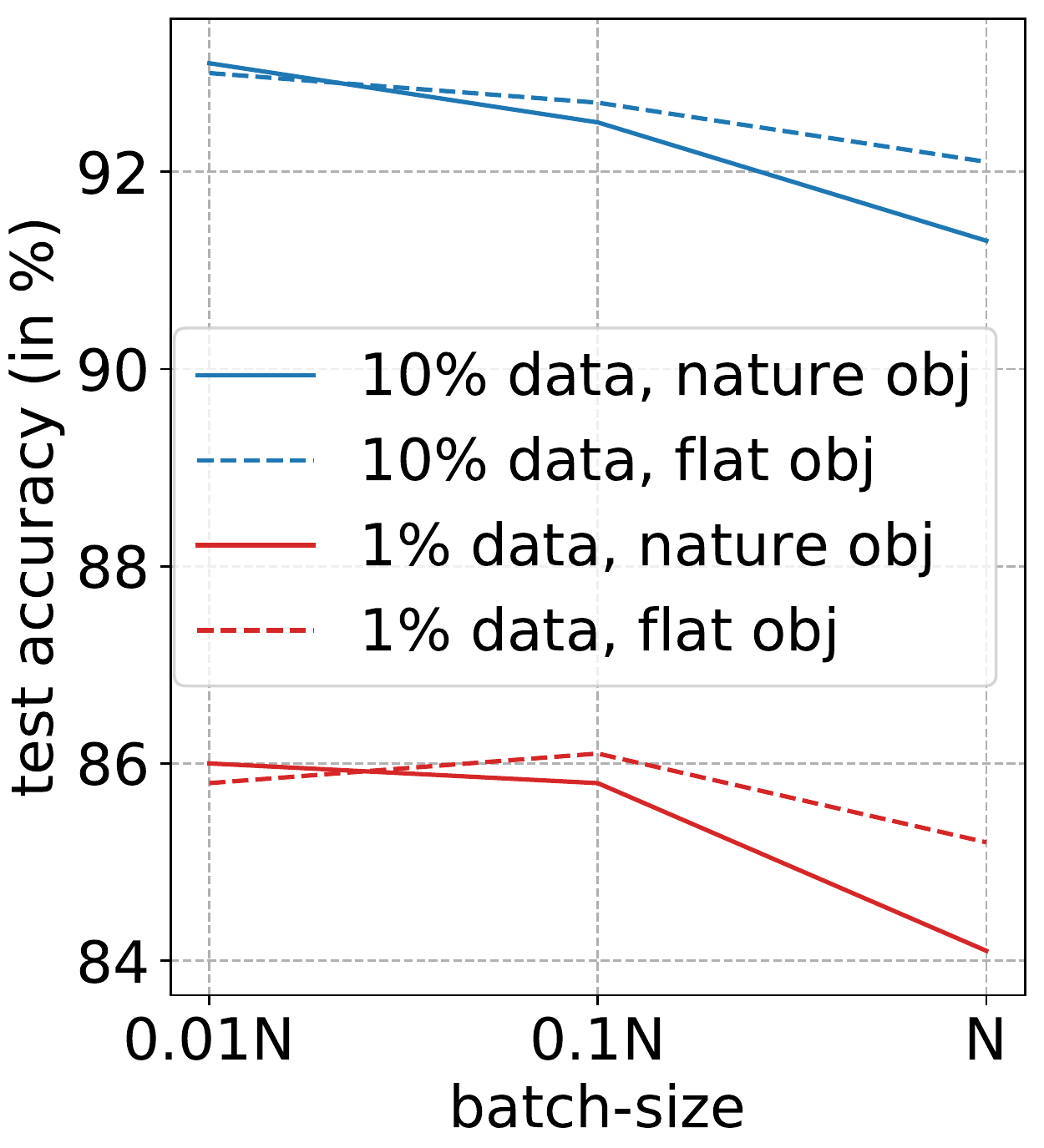}
 &
 \hspace{-1mm}\includegraphics[width=.45\textwidth]{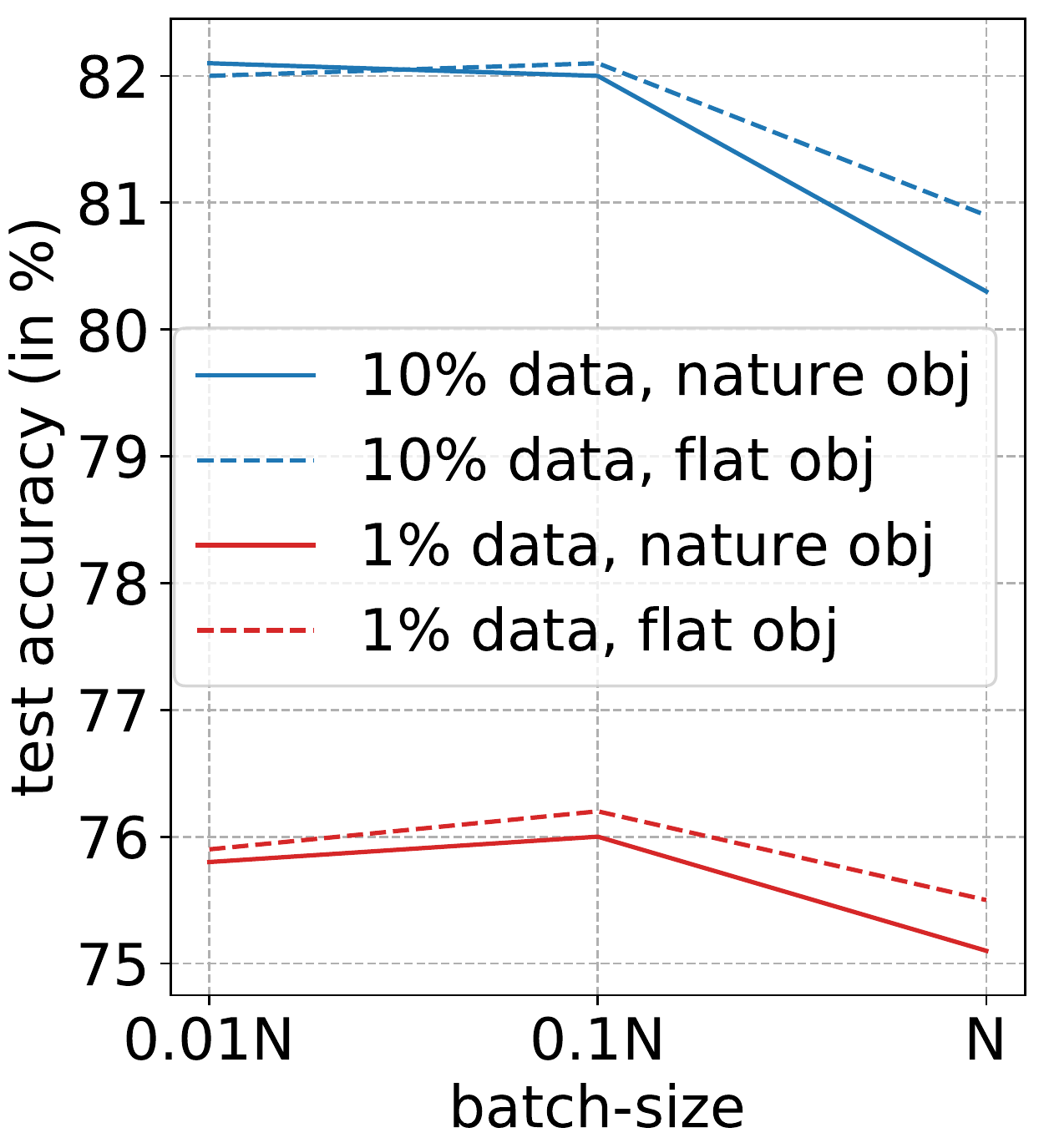} \\
 \footnotesize{ MNIST} &\footnotesize{ FashionMNIST}
 \end{tabular}
 \caption{}
 \label{fig: flat_acc}
\end{subfigure}
\hspace{-2mm}
\begin{subfigure}{0.45\linewidth}
   \centering
   \begin{tabular}{cc}
 \includegraphics[width=.6\textwidth]{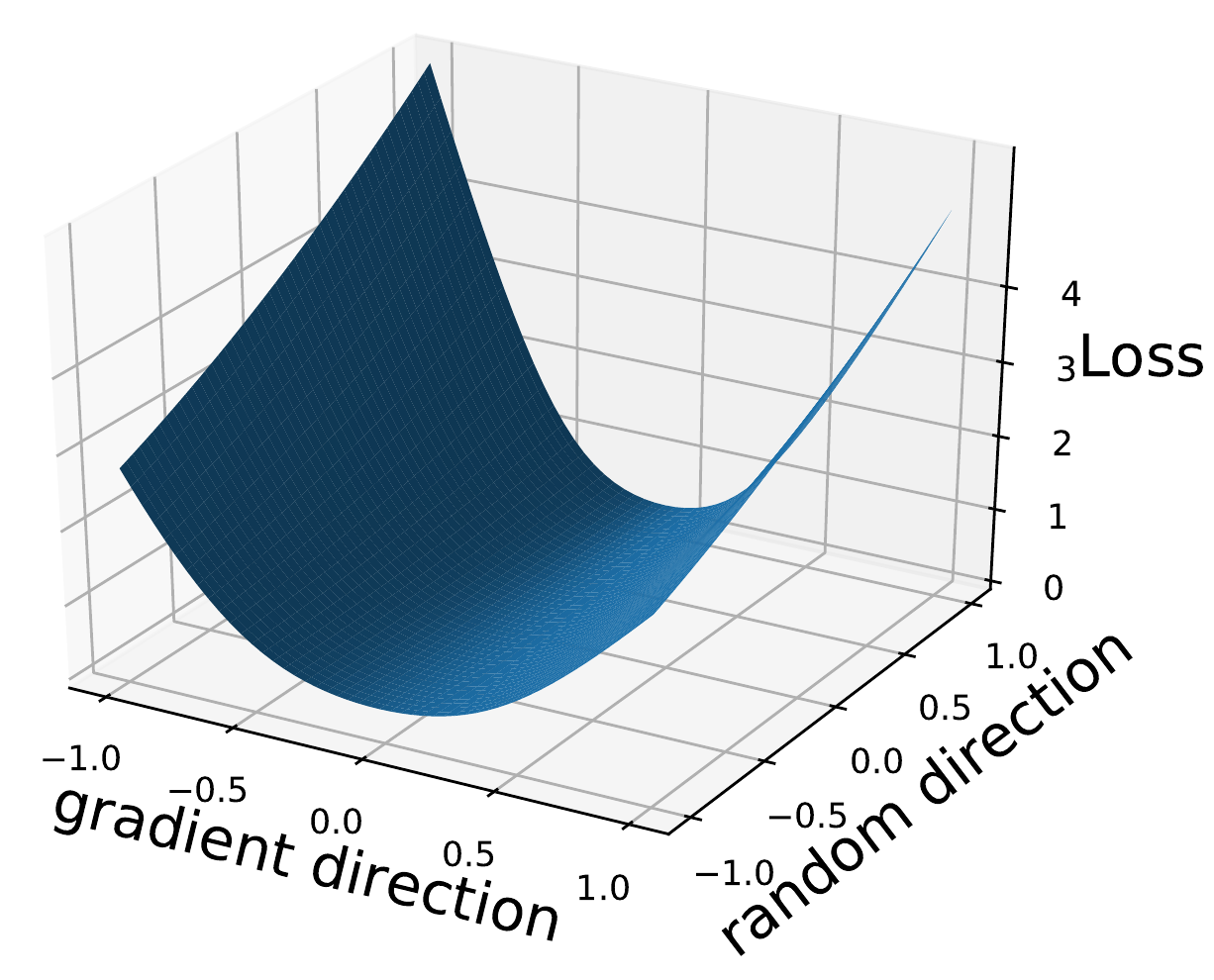} &
 \hspace{-5mm}
 \includegraphics[width=.6\textwidth]{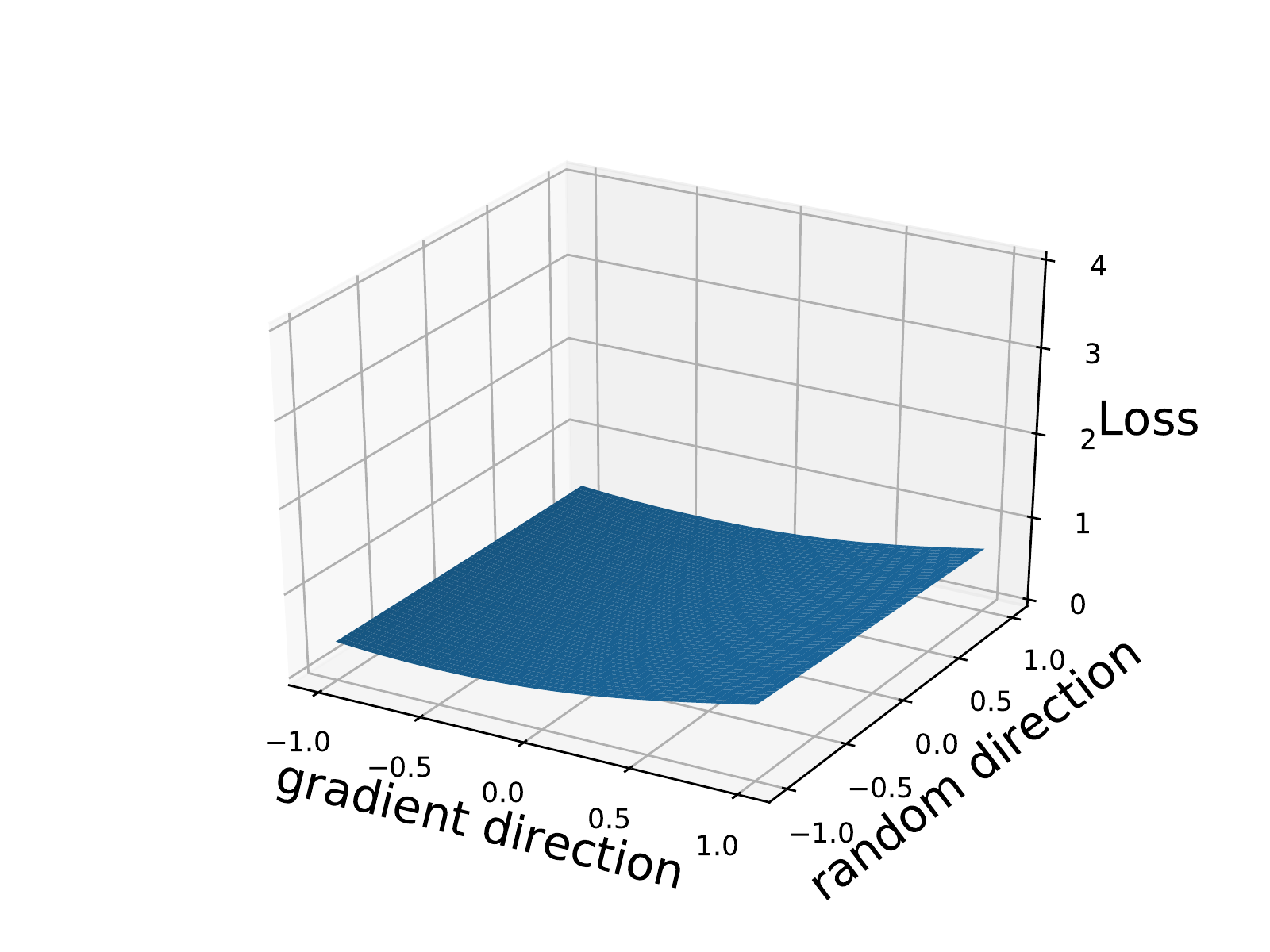}
\\
\footnotesize{ Nature model} & \hspace{-5mm} \footnotesize{ Flat  model}
 \end{tabular}
 \vspace{-1mm}
  \caption{}
  \label{fig: optimization_loss}
\end{subfigure}
\vspace{-5pt}
\caption {Application of applying LiRPA bounds to network parameters to obtain a model with a provably ``flat'' loss surface. (a) Test accuracy of naturally trained models and ``flat'' objective trained models  on MNIST and FashionMNIST with different combinations of data size and batch size. 
(b) The training loss landscape of models trained with nature and flat objective on 10\% data of MNIST with $0.1 N$ batch size. We plot the loss surface along the gradient direction and a random direction.}
\end{figure}  

Most previous works on LiRPA based certified defense only implemented input perturbations analysis. Our framework naturally extends to perturbation analysis on network parameters $\theta$
as they are also independent nodes in a computational graph (e.g., node $x_2$ in Figure~\ref{fig:backward}). This requires to relax the multiplication operation (e.g., the MatMul nodes in Figure~\ref{fig:backward}) which was first discussed in~\citet{shi2020robustness}, and our Algorithm~\ref{alg:backward} can then be directly applied.  
With this advantage, LiRPA can compute provable upper and lower bounds on the local ``flatness'' around a certain point $\theta_0$ for some loss $\mathcal{L}$:
\begin{equation}
\label{eq:bound_loss_landscape}
\small{
\mathcal{L}(\theta_0) - C_L(\theta_0) \leq \mathcal{L}(\theta_0 + \Delta \theta) \leq \mathcal{L}(\theta_0) + C_U(\theta_0), \enskip \text{for all} \enskip \|\Delta \theta\|_2 \leq \epsilon,}
\end{equation}
where $C_L$ and $C_U$ are linear functions of $\theta_0$ that can be found using LiRPA. This is a ``zeroth order'' flatness criterion, where we guarantee that the loss value does not change too much in a small region around $\theta_0$, and we do not have further assumptions on gradients or Hessian of the loss. 
When $\theta_0$ is a good solution, $\mathcal{L}(\theta_0)$ is close to 0, so we can simply set the left hand side of~\eqref{eq:bound_loss_landscape} to 0 and upper bound $\mathcal{L}(\theta_0 + \Delta \theta)$ to ensure flatness.
Using our framework, we can train a classifier that guarantees flatness of local optimization landscape, 
by minimizing the ``flat'' objective $\mathcal{L}(\theta_0) + C_U(\theta_0)$ for the perturbation set $\sS(\theta_0)\!=\!\{\theta\!:\!\| \theta - \theta_0 \|_2 \leq \epsilon\}$ where $\theta_0$ is the current network parameter. When this ``flat'' objective is close to 0, we guarantee that $\mathcal{L}$ is close to 0 for all $\theta \in \sS(\theta_0)$. 
We build a three-layer MLP model with $[64, 64, 10]$ neurons in each layer and conduct experiments using only $10\%$ and $1\%$ of the training data in MNIST and FashionMNIST, and we then test on the full test set to aggressively evaluate the generalization performance. 
We also aggressively set the batch size to $\{0.01N, 0.1N, N\}$ as in~\cite{jastrzebski2018finding} where $N$ is the size of training dataset. Additional details can be found in Appendix~\ref{sec:app_detail_flat}.

The test accuracies of the models trained with regular cross entropy and our ``flat'' objective are shown in Figure\,\ref{fig: flat_acc}. We visualize their loss surfaces in Figure~\ref{fig: optimization_loss}. When batch size is increased or fewer data are used, test accuracy generally decreases due to overfitting, which is consistent with \cite{keskar2016large}. 
For models trained with the flat objective, the accuracy tends to be better, especially when a very large batch size is used. 
These observations provide some evidence for the hypothesis that a flat local minimum generalizes better, however, we cannot exclude the possibility that the improvements come from side effects of our objective. 
Our focus is to demonstrate potential applications beyond neural network verification of our framework rather than proving this hypothesis.

%% file: appendix.tex
\noindent 
In Appendix~\ref{apd:additional_discussions}, we provide more discussions on LiRPA bounds, including detailed algorithm and complexity analysis, comparison of different LiRPA implementations, and also a small numerical example in Appendix~\ref{sec:small_example}. 
In Appendix~\ref{apd:proofs}, we provide proofs of the theorems.
We provide additional experiments, including more LiRPA trained TinyImageNet models and IBP baselines in Appendix~\ref{sec:app_vision_training}, and we also provide details for each experiment in Appendix~\ref{apd:additional_experiments}.

\section{Additional Discussions on LiRPA Bounds}
\label{apd:additional_discussions}
\subsection{Oracle Functions and the Linear Relaxation of Nonlinear Operations}
\label{apd:functions}

In this section, we summarize some examples of oracle functions as derived in previous works~\citep{zhang2018efficient,wang2018efficient,shi2020robustness}.
In Table~\ref{table:oracle_functions}, we provide a list of oracle functions of three basic operation types, including affine transformation, unary nonlinear function, and binary nonlinear function.
Most common operations involved in neural networks can be addressed following these basic operation types.
For example, dense layers and convolutional layers are affine transformations, activation functions are unary nonlinear functions, multiplication and division are binary nonlinear functions, and matrix multiplication or dot product of two variable matrices can be considered as multiplications with an affine transformation.

Parameters $\underline{\alpha}, \underline{\beta}, \underline{\gamma},\overline{\alpha},\overline{\beta},\overline{\gamma}$ in Table~\ref{table:oracle_functions} are involved in the linear relaxation of nonlinear operations.
For example, for ReLU, $\sigma(h_j(\rmX))=\max(h_j(\rmX),0)$, is a piecewise linear function and can be linearly relaxed w.r.t. the bounds of $h_j(\rmX)$, denoted as $l\leq h_j(\rmX)\leq u$.
When $u\leq 0$ or $l\geq 0$, $\sigma(h_j(\rmX))$ is a linear function on $h_j(\rmX)\in[l,u]$, and thus $\sigma(h_j(\rmX))=h_j(\rmX)$ is a linear function, i.e., we can take $\underline{\alpha}=\overline{\alpha}=1, \underline{\beta}=\overline{\beta}=0$.
Otherwise, for $l<0<u$, we can take the line passing $(l,\sigma(l))$ and $(u,\sigma(u))$ as the linear upper bound, i.e., $\overline{\alpha}=\frac{\sigma(u)-\sigma(l)}{u-l}$, $\overline{\beta}=-\overline{\alpha}l$.
For the lower bound, it can be any line with $0\leq \underline{\alpha}\leq 1$ and $\underline{\beta}=0$.
To minimize the relaxation error, \citet{zhang2018efficient} proposed to adaptively choose $\underline{\alpha}=I(u>|l|)$ in LiRPA.
Alternatively, we can also select $\underline{\alpha}=0$, and thereby the linear relaxation can be provably tighter than IBP bounds.
This lower bound can be used for training ReLU networks with loss fusion.
Figure~\ref{fig:relu_relaxation} compares the linear bounds in LiRPA and IBP respesctively.

\begin{figure}[ht]
    \centering
    \includegraphics[width=.6\textwidth]{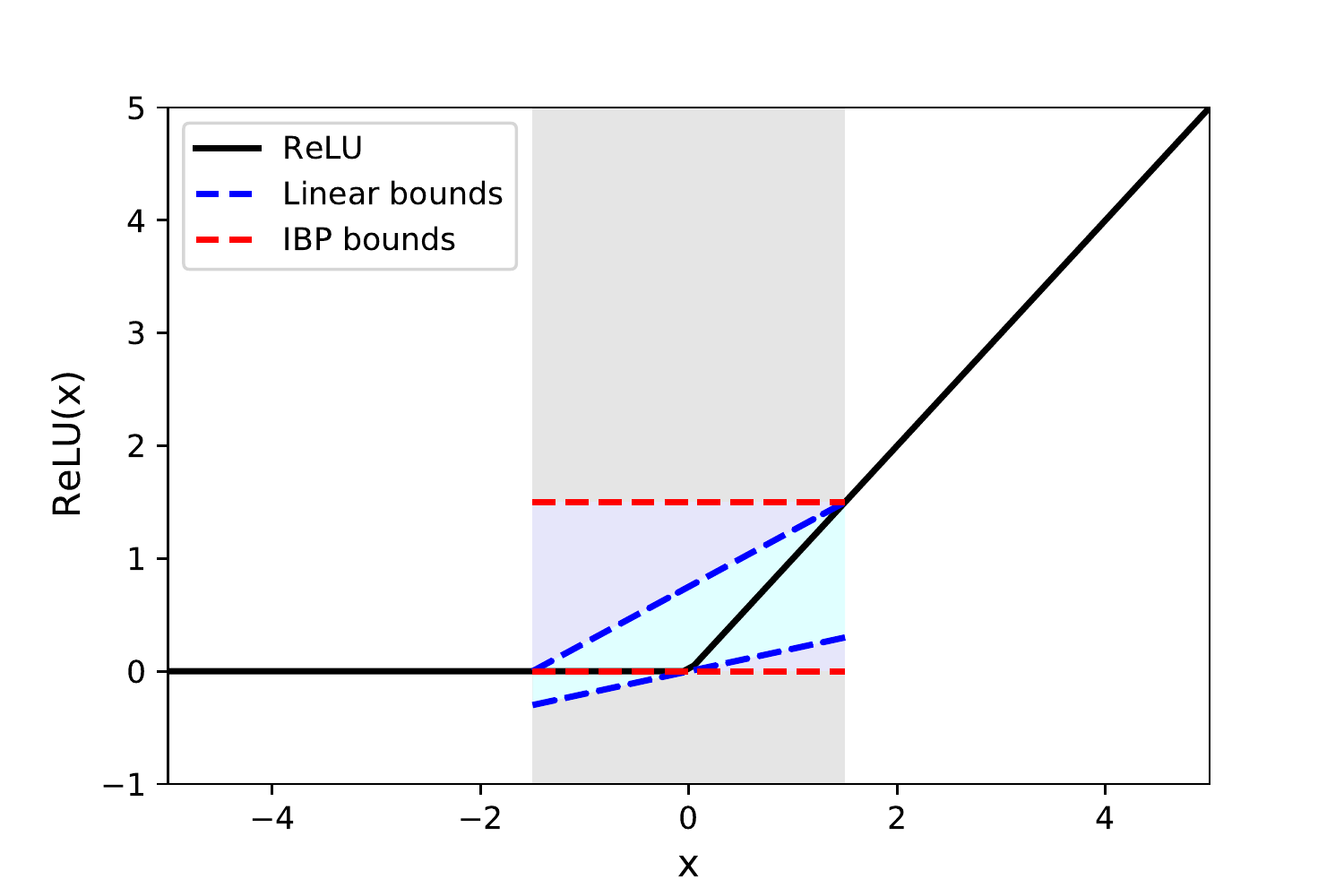}
    \caption{An example of ReLU relaxation when $l=-1.5,\,u=1.5$. 
    Here we take the blue dashed lines as the linear bounds, where any line passing $(0,0)$ with a slope between 0 and 1 can be a valid lower bound.
    In contrast, IBP takes the fixed red dashed lines as the lower and upper bounds respectively, which is a looser relaxation.}
    \label{fig:relu_relaxation}
\end{figure}

The detailed derivation of the oracle functions shown in Table~\ref{table:oracle_functions} has been covered in previous works~\citep{zhang2018efficient,wang2018efficient,shi2020robustness} and is not a focus of this paper.
We refer readers to those existing works for details.

\begin{table}[!ht]
  \centering
    \caption{A list of common types of operations, their definition $H_i$, and their corresponding oracle functions $F_i$ and $G_i$.
    Subscript ``+'' stands for taking positive elements from the matrix or vector while setting other elements to zero, and vice versa for subscript ``-''.
    $\text{diag}(\cdot)$ stands for constructing a diagonal matrix from a vector.
    $\underline{\alpha},\underline{\beta},\underline{\gamma},\overline{\alpha},\overline{\beta},\overline{\gamma}$ are parameters of linear relaxation that can be derived for each specific nonlinear function.
   }
  \adjustbox{max width=\textwidth}{
  \begin{tabular}{c|cl}
  \toprule[1pt]
   Operation Type & \multicolumn{2}{c}{Functions}\\
   \hline
   
   \multirow{9}{*}{Affine Transformation} &
   $H_i$ &  
   $h_i(\rmX) = \hat{\rmW}_i h_j(\rmX) + \hat{\rvb}_i$\\
   \cline{2-3}
   & \multirow{4}{*}{$F_i$} &$\underline{\newA}_j=\underline{\rmA}_i\hat{\rmW}_i$ \\
   & & $\overline{\newA}_j=\overline{\rmA}_i\hat{\rmW}_i$\\
   & & $\underline{\newd}=\underline{\rmA}_i\hat{\rvb}_i$\\
   & & $\overline{\newd}=\overline{\rmA}_i\hat{\rvb}_i$\\
   \cline{2-3}
   & \multirow{4}{*}{$G_i$} & $\underline{\rmW}_i = \hat{\rmW}_{i,+} \underline{\rmW}_j + \hat{\rmW}_{i,-} \overline{\rmW}_j$\\
   & &  $\underline{\rvb}_i = \hat{\rmW}_{i,+}\underline{\rvb}_j + \hat{\rmW}_{i,-} \overline{\rvb}_j + \hat{\rvb}_i$\\
   & & $\overline{\rmW}_i = \hat{\rmW}_{i,+} \overline{\rmW}_j + \hat{\rmW}_{i,-} \underline{\rmW}_j$\\
   & & $\overline{\rvb}_i = \hat{\rmW}_{i,+}\overline{\rvb}_j + \hat{\rmW}_{i,-} \underline{\rvb}_j + \hat{\rvb}_i$\\
   \hline

   \multirow{9}{*}{Unary Nonlinear Function} &
   $H_i$ &  
   $h_i(\rmX) = \sigma(h_j(\rmX))$\\
   \cline{2-3}
   & \multirow{4}{*}{$F_i$} &$\underline{\newA}_j=\underline{\rmA}_{i,+} \text{diag}(\underline{\alpha}) +\underline{\rmA}_{i,-} \text{diag}(\overline{\alpha}) $ \\
   & & $\overline{\newA}_j=\overline{\rmA}_{i,+} \text{diag}(\overline{\alpha}) +\overline{\rmA}_{i,-} \text{diag}(\underline{\alpha}) $\\
   & & $\underline{\newd}=\underline{\rmA}_{i,+} \underline{\beta}+ \underline{\rmA}_{i,-} \overline{\beta}$\\
   & & $\overline{\newd}=\overline{\rmA}_{i,+} \overline{\beta}+ \overline{\rmA}_{i,-} \underline{\beta}$\\
   \cline{2-3}
   & \multirow{4}{*}{$G_i$} & $\underline{\rmW}_i =  \text{diag}_+(\underline{\alpha}) \underline{\rmW}_j + \text{diag}_-(\underline{\alpha}) \overline{\rmW}_j    $\\
   & &  $\underline{\rvb}_i =  \text{diag}_+(\underline{\alpha}) \underline{\rvb}_j + \text{diag}_-(\underline{\alpha}) \overline{\rvb}_j +\underline{\beta}$\\
   & & $\overline{\rmW}_i =  \text{diag}_+(\overline{\alpha}) \overline{\rmW}_j + \text{diag}_-(\overline{\alpha}) \underline{\rmW}_j $\\
   & & $\overline{\rvb}_i =\text{diag}_+(\overline{\alpha}) \overline{\rvb}_j + \text{diag}_-(\overline{\alpha}) \underline{\rvb}_j +\overline{\beta} $\\
   \cline{2-3}
   & where & $ \underline{\alpha} h_j(\rmX) + \underline{\beta} \leq h_i(\rmX) \leq \overline{\alpha} h_j(\rmX) + \overline{\beta} $ 
   \\
   \hline
   
   \multirow{12}{*}{Binary Nonlinear Function} &
   $H_i$ &  
   $h_i(\rmX) = \pi(h_j(\rmX),h_k(\rmX))$\\
   \cline{2-3}
   & \multirow{6}{*}{$F_i$} &$\underline{\newA}_j=\underline{\rmA}_{i,+} \text{diag}(\underline{\alpha}) +\underline{\rmA}_{i,-} \text{diag}(\overline{\alpha}) $ \\
   & & $\overline{\newA}_j=\overline{\rmA}_{i,+} \text{diag}(\overline{\alpha}) +\overline{\rmA}_{i,-} \text{diag}(\underline{\alpha}) $\\
   & &$\underline{\newA}_k=\underline{\rmA}_{i,+} \text{diag}(\underline{\beta}) +\underline{\rmA}_{i,-} \text{diag}(\overline{\beta}) $ \\
   & & $\overline{\newA}_k=\overline{\rmA}_{i,+} \text{diag}(\overline{\beta}) +\overline{\rmA}_{i,-} \text{diag}(\underline{\beta}) $\\
   & & $\underline{\newd}=\underline{\rmA}_{i,+} \underline{\gamma}+ \underline{\rmA}_{i,-} \overline{\gamma}$\\
   & & $\overline{\newd}=\overline{\rmA}_{i,+} \overline{\gamma}+ \overline{\rmA}_{i,-} \underline{\gamma}$\\
   \cline{2-3}
   & \multirow{4}{*}{$G_i$} & $\underline{\rmW}_i =  \text{diag}_+(\underline{\alpha}) \underline{\rmW}_j + \text{diag}_-(\underline{\alpha}) \overline{\rmW}_j+
   \text{diag}_+(\underline{\beta}) \underline{\rmW}_k + \text{diag}_-(\underline{\beta}) \overline{\rmW}_k
   $\\
   & &  $\underline{\rvb}_i =  \text{diag}_+(\underline{\alpha}) \underline{\rvb}_j + \text{diag}_-(\underline{\alpha}) \overline{\rvb}_j +\underline{\beta}+ \text{diag}_+(\underline{\beta}) \underline{\rvb}_k + \text{diag}_-(\underline{\beta}) \overline{\rvb}_k +\underline{\gamma}   $\\
   &  & $\overline{\rmW}_i =  \text{diag}_+(\overline{\alpha}) \overline{\rmW}_j + \text{diag}_-(\overline{\alpha}) \underline{\rmW}_j +
   \text{diag}_+(\overline{\beta}) \overline{\rmW}_k + \text{diag}_-(\overline{\beta}) \underline{\rmW}_k$\\
   & &  $\overline{\rvb}_i =  \text{diag}_+(\overline{\alpha}) \overline{\rvb}_j + \text{diag}_-(\overline{\alpha}) \underline{\rvb}_j +\overline{\beta}+ \text{diag}_+(\overline{\beta}) \overline{\rvb}_k + \text{diag}_-(\overline{\beta}) \underline{\rvb}_k +\overline{\gamma}   $\\
   \cline{2-3}
   & where & $ \underline{\alpha} h_j(\rmX) +\underline{\beta} h_k(\rmX) + \underline{\gamma} \leq h_i(\rmX) \leq \overline{\alpha} h_j(\rmX) +\overline{\beta} h_k(\rmX) + \overline{\gamma}$\\   
   \bottomrule[1pt]
  \end{tabular}
  }
  \label{table:oracle_functions}
\end{table}  

\subsection{Complexity Comparison between Different Perturbation Analysis Modes}
\label{apd:complexity}

In this section, we compare the computational cost of different perturbation analysis modes.
We assume that $D_x$ and $D_y$ are the total dimension of the perturbed independent nodes and the final output node respectively.
We focus on a usual case in classification models, where the final output node is a logits layer whose dimension equals to the number of classes and thus usually $D_y\ll D_x$ holds true, or the final output is a loss function with $D_y=1\ll D_x$ if loss fusion is enabled.
We also assume that the time complexity of a regular forward pass of the computational graph (e.g., a regular inference pass) is $O(r)$, and the complexity of a regular back propagation pass in gradient computation is also asymptotically $O(r)$. Note that the overall time complexity of LiRPA depends on oracle functions, and in the below analysis we focus on common cases (e.g., common activation functions in Table~\ref{table:oracle_functions}). 

\paragraph{Interval bound propagation (IBP)} IBP can be seen as a special and degenerated case of LiRPA bounds. The time complexity of pure IBP is still $O(r)$ since it computes two output values, a lower bound and a upper bound, for each neuron, and thus the time complexity is the same as a regular forward pass which computes one output value for each neuron. 
However, pure IBP cannot give tight enough bounds especially for models without certifiably robust training.

\paragraph{Backward mode bound propagation} Backward mode LiRPA oracles typically require bounds of intermediate nodes $\underline{\rvh}_j$, $\overline{\rvh}_j$ for all $j \in u(i)$ for a node $i$ (referred to as ``pre-activation bounds'' in some works). Assuming these intermediate bounds are known; in this case, the oracle function $F_i$ typically has the same time complexity as back propagation of gradients through node $i$ (e.g., for linear layers it is the transposed operation of $H_i (\cdot)$). However, unlike in back propagation where the gradients is computed for a scalar function, in backward mode LiRPA we need to compute $O(D_y)$ values for each neuron, and these values stand for the coefficients of the linear bounds of the $D_y$ final output neurons. The time complexity is roughly $D_y$ times back propagation time, $O(D_y r)$. 

For a purely backward perturbation analysis that can be extended from CROWN~\citep{zhang2018efficient}, the bounds of intermediate nodes needed for the oracle functions are also computed with a backward mode LiRPA. Assuming there are $N$ nodes in total (including output nodes and all intermediate nodes) that require LiRPA bounds, the total time complexity is asymptotically $O(N r)$ where $N$ can be a quite large number (e.g., for feed-forward ReLU networks $N$ includes hidden neurons over all layers and $N \gg D_y$), so this approach cannot scale to large graphs or be used for efficient training.

\paragraph{Forward mode bound propagation}
In the forward mode perturbation analysis, 
since we represent the bounds of each neuron with linear functions w.r.t. the perturbed independent nodes, we need to compute $O(D_x)$ values for each neuron. Usually, the oracle functions $G_i$ has the same asymptotic complexity as the computation function $H_i(\cdot)$; however, the inputs of $G_i$ include dimension $D_x$, and the total time complexity of is roughly $O(D_x r)$.
Note that in the implementation of the forward mode, we do not compute linear functions w.r.t. all the independent nodes, but we only need to consider those perturbed independent nodes while treating the other independent nodes as constants, and thereby $D_x$ may be much smaller than the dimension of $\rmX$, e.g., model parameters can be excluded if they are not perturbed.

\paragraph{Efficient hybrid bounds} Among the LiRPA variants, \emph{IBP+Backward} with a complexity of $O(D_y r)$ is usually most efficient for classification models and is used in our certified training experiments. When loss fusion is enabled, $D_y=1$ during training, and thereby the complexity of \emph{IBP+Backward} is  $O(r)$, which is the same as that of IBP.
In this way, our loss fusion technique can significantly improve the scalability of certified training with LiRPA bounds.
To obtain tighter bounds for intermediate nodes which can also tighten the final output bounds, we may use pure forward or \emph{Forward+Backward} mode with a complexity of $O(D_x r)$ which is usually larger than that of \emph{IBP+Backward} when $D_y\ll D_x$.
The forward mode LiRPA can be potentially useful for situations where $D_x\ll D_y$, e.g., for generative models with a large output dimension. We leave this as our future work.

\subsection{The GetOutDegree Auxiliary Function in Backward Mode Perturbation Analysis}
\label{apd:getdegree}

\begin{algorithm}[ht]
    \caption{Auxiliary Function for Computing Output Degrees}
    \begin{algorithmic}
        \FUNCTION{GetOutDegree~($o$)}
            \STATE Create BFS queue and $Q.push(o)$
            \STATE $d_i\leftarrow 0\ \  (\forall i\leq n)$
            \WHILE{$Q$ is not empty}
                \STATE $i = Q.pop()$
                \FOR{$j\in u(i)$}
                    \STATE $d_{j}+\!\!=1$
                    \IF{$j$ has not been in $Q$}\STATE{$Q.push(j)$}
                    \ENDIF
                \ENDFOR
            \ENDWHILE
        \ENDFUNCTION 
    \end{algorithmic}
    \label{alg:get_degree}
\end{algorithm}

As mentioned in Section 3.4, we have an auxiliary ``GetOutDegree'' function for computing the degree $d_i$ of each node $i$, which is defined as the the number of outputs nodes of node $i$ that the node $o$ is dependent on. 
This function is illustrated in Algorithm \ref{alg:get_degree}. 
We use a BFS pass. 
At the beginning, node $o$ is added into the queue.
Next, each time we pick a node $i$ from the head of the queue.
Node $o$ is dependent on node $i$, and thus we increase the degree of its input nodes, each $d_j(j\in u(i))$, by 1.
Node $o$ is also dependent on node $j(j\in u(i))$ and we add node $j$ to the queue if it has never been in the queue yet.
We repeat this process until the queue becomes empty, and at this time any node $i$ that node $o$ is dependent on has been visited and has contributed to the $d_{j}(j\in u(i))$ of its input nodes.

\subsection{A Small Example of LiRPA Bounds}
\label{sec:small_example}
We provide a small example to illustrate the computation of our LiRPA methods. We assume that we have a simple ReLU network with 2 hidden layers, with weight matrix of each layer as below:
\begin{equation*}
    \hat{\rmW}_1 = [[2, 1], [-3, 4]], \enskip
    \hat{\rmW}_2 = [[4, -2], [2, 1]], \enskip
    \hat{\rmW}_3 = [-2, 1],
\end{equation*}
and we do not consider bias terms of the layers here for simplicity.

Given a clean input $\rmX_0 = [[0], [1]]$ and $\ell_\infty$ perturbation with $\eps = 2$, we can compute the bounds of the last layer and compare the results from IBP, forward mode LiRPA and backward mode LiRPA respectively. 

\paragraph{IBP} 
\begin{equation*}
\begin{aligned}
\overline{\rvh}_1 &= [[2], [3]], \\
\underline{\rvh}_1 &= [[-2], [-1]],\\
    \overline{\rvh}_2 &= \hat{\rmW}_{1, +}\overline{\rvh}_1  +  \hat{\rmW}_{1, -}\underline{\rvh}_1 = [[7], [12]] + [[0], [6]] = [[7], [18]],\\
    \underline{\rvh}_2 &= \hat{\rmW}_{1, +} \underline{\rvh}_1  +  \hat{\rmW}_{1, -} \overline{\rvh}_1= [[-5], [-4]] + [[0], [-6]] = [[-5], [-10]],\\
    \overline{\rvh}_3 &=  \hat{\rmW}_{2, +}\overline{\rvh}_2 +  \hat{\rmW}_{2, -}\underline{\rvh}_2 = [[28], [32]] + [[0], [0]]= [[28], [32]],\\
    \underline{\rvh}_3 &=  \hat{\rmW}_{2, +}\underline{\rvh}_2 +  \hat{\rmW}_{2, -}\overline{\rvh}_2 = [[0], [0]] + [[-36], [0]] = [[-36], [0]],\\
    \overline{\rvh}_4 &= \hat{\rmW}_{3, +}\overline{\rvh}_3 + \hat{\rmW}_{3, -}\underline{\rvh}_3 = [32] + [0] = [32],\\
    \underline{\rvh}_4&= \hat{\rmW}_{3, +}\underline{\rvh}_3 + \hat{\rmW}_{3, -}\overline{\rvh}_3 = [0] + [-56] = [-56].
\end{aligned}
\end{equation*}

In the following computation of LiRPA bounds, we always use zero as the lower bound of ReLU activation.

\paragraph{Forward Mode LiRPA} 

\begin{equation*}
    \begin{aligned}
    	\overline{\rmW}_1&=\underline{\rmW}_1=\rmI, \enskip\underline{\rvb}_1=\overline{\rvb}_1=\vzero,\\
        \overline{\rmW}_2 &= \underline{\rmW}_2 = \hat{\rmW}_1 = [[2, 1], [-3, 4]],\\
        \overline{\rvh}_2 &= 2[[3], [7]] + [[1], [4]] = [[7], [18]],\\
        \underline{\rvh}_2 &= -2[[3], [7]] + [[1], [4]] = [[-5], [-10]].
    \end{aligned}
\end{equation*}

We compute the relaxation of the first layer ReLU activations:
\begin{equation*}
    \begin{aligned}
        \text{diag}(\overline{\alpha}_1) &= [[0.58, 0], [0, 0.64]], \\\text{diag}(\underline{\alpha}_1) &= [[0, 0], [0, 0]],\\
        \overline{\beta_1} &= [[2.92], [6.43]]], \\\underline{\beta_1} &= [[0], [0]],\\
    \end{aligned}
\end{equation*}
and then we have:
\begin{equation*}
    \begin{aligned}
        \overline{\rmW}_3 &= \hat{\rmW}_{2, +}(\text{diag}(\overline{\alpha}_1) \overline{\rmW}_2)  + \hat{\rmW}_{2, -} (\text{diag}(\underline{\alpha}_1) \underline{\rmW}_2)= [[4.67, 2.33], [0.40, 3.74]],\\
        \underline{\rmW}_3 &= \hat{\rmW}_{2, -} (\text{diag}(\overline{\alpha}_1) \overline{\rmW}_2) + \hat{\rmW}_{2, +} (\text{diag}(\underline{\alpha}_1) \underline{\rmW}_2) = [[3.86, -5.14], [0, 0]],\\
        \overline{\rvd}_2 &= \hat{\rmW}_{2, +} \overline{\beta}_1 + \hat{\rmW}_{2, -} \underline{\beta}_1 = [[11.67],[12.26] ], 
        \\
        \underline{\rvd}_2 &=\hat{\rmW}_{2, -} \overline{\beta}_1 + \hat{\rmW}_{2, +} \underline{\beta}_1 = [[-12.86], [0]],\\
        \overline{\rvh}_3 &= \underline{\rmW_3} \rmX_0 + \lVert \underline{\rmW_3} \rVert_1 \epsilon + \underline{\rvd_2} =[[28], [24]],\\
        \underline{\rvh}_3 &= \underline{\rmW_3} \rmX_0 + \lVert \underline{\rmW_3} \rVert_1 \epsilon + \underline{\rvd_2} =[[-36], [0]].
    \end{aligned}
\end{equation*}
We then repeat the computation on the second layer:
\begin{equation*}
    \begin{aligned}
        \text{diag}(\overline{\alpha}_2) &= [[0.4375, 0], [0, 1]],\\ \text{diag}(\underline{\alpha}_2) &= [[0, 0], [0, 1],]\\
        \overline{\beta_2} &= [[15.75], [0]], \\
        \underline{\beta_2} &= [[0], [0]],\\
    \end{aligned}
\end{equation*}
\begin{equation*}
    \begin{aligned}        
        \overline{\rmW}_4 &= \hat{\rmW}_{3, +} (\text{diag}(\overline{\alpha}_2) \overline{\rmW}_3) + \hat{\rmW}_{3, -} (\text{diag}(\underline{\alpha}_2) \underline{\rmW}_3)= [0.40, 3.74],\\
        \underline{\rmW}_4 &= \hat{\rmW}_{3, -} (\text{diag}(\overline{\alpha}_2) \overline{\rmW}_3) + \hat{\rmW}_{3, +} (\text{diag}(\underline{\alpha}_2) \underline{\rmW}_3) = [-4.08, -2.04] ,\\
        \overline{\rvd}_3 &= \hat{\rmW}_{3, +}(\overline{\beta}_2 + \text{diag}(\overline{\alpha}_2)\overline{\beta}_2)+ \hat{\rmW}_{3, -}(\underline{\beta}_2 + \text{diag}(\overline{\alpha}_2)\underline{\beta}_2) = [12.26],\\
        \underline{\rvd}_3 &= \hat{\rmW}_{3, -}(\overline{\beta}_2 + \text{diag}(\overline{\alpha}_2)\overline{\beta}_2)+ \hat{\rmW}_{3, +}(\underline{\beta}_2 + \text{diag}(\overline{\alpha}_2)\underline{\beta}_2) =[-41.71],\\
        \overline{\rvh}_4 &= \overline{\rmW}_4 \rmX_0 + \lVert \overline{\rmW}_4 \rVert_1 \epsilon + \overline{\rvd}_3 = [24.29],\\
        \underline{\rvh}_4 &= \underline{\rmW}_4 \rmX_0 + \lVert \underline{\rmW}_4 \rVert_1 \epsilon + \underline{\rvd}_3 = [-56].
    \end{aligned}
\end{equation*}

\paragraph{Backward Mode LiRPA} Here we reuse the intermediate results from the forward mode LiRPA for the linear relaxation of ReLU activations, where  
\begin{equation*}
    \begin{aligned}
        \text{diag}(\overline{\alpha}_1) &= [[0.58, 0], [0, 0.64]],\\ \text{diag}(\underline{\alpha}_1) &= [[0, 0], [0, 0]],\\
        \overline{\beta_1} &= [[2.92], [6.43]]], \\
        \underline{\beta}_1 &= [[0], [0]],\\
        \text{diag}(\overline{\alpha}_2) &= [[0.4375, 0], [0, 1]],\\ \text{diag}(\underline{\alpha}_2) &= [[0, 0], [0, 1]]\\
        \overline{\beta}_2 &= [[15.75], [0]],\\
        \underline{\beta}_2 &= [[0], [0]].
    \end{aligned}
\end{equation*}

We then compute the linear bounds from the last layer to the first layer and finally concretize the linear bounds:
\begin{equation*}
    \begin{aligned}
        \underline{\rmA}_4 &= \overline{\rmA}_4=\rmI,\\
        \underline{\rmA}_3 &=\underline{\rmA}_4 \hat{\rmW}_3 = [-2, 1],\\
        \overline{\rmA}_3 &=\overline{\rmA}_4 \hat{\rmW}_3 = [-2, 1],\\
        \overline{\rmA}_2 &=
        \overline{\rmA}_{3,+} \text{diag}(\overline{\alpha}_2)\hat{\rmW}_2+
        \overline{\rmA}_{3,-}\text{diag}(\underline{\alpha}_2)\hat{\rmW}_2 = [2, 1],\\
        \underline{\rmA}_2 &= 
        \underline{\rmA}_{3,+} \text{diag}(\underline{\alpha}_2)\hat{\rmW}_2+
        \underline{\rmA}_{3,-}\text{diag}(\overline{\alpha}_2)\hat{\rmW}_2 = [-1.5, 2.75],\\
        \overline{\rmA}_1 &= 
        \overline{\rmA}_{2,+} \text{diag}(\overline{\alpha}_1)\hat{\rmW}_1+
        \overline{\rmA}_{2,-}\text{diag}(\underline{\alpha}_1)\hat{\rmW}_1 = [0.40, 3.74],\\
        \underline{\rmA}_1 &=
        \underline{\rmA}_{2,+} \text{diag}(\underline{\alpha}_1)\hat{\rmW}_1+
        \underline{\rmA}_{2,-}\text{diag}(\overline{\alpha}_1)\hat{\rmW}_1 = [-1.75, -0.875],\\
        \overline{\rvd}_1 &= \overline{\rmA}_{2, +} \overline{\beta}_2 + \overline{\rmA}_{2, -} \underline{\beta_2} + \overline{\rmA}_{1, +} \overline{\beta_1} + \overline{\rmA}_{1, -} \underline{\beta_1} = [12.26],\\
        \underline{\rvd}_1 &= \underline{\rmA}_{2, +} \underline{\beta}_2 + \underline{\rmA}_{2, -}\overline{\beta}_2 + \underline{\rmA}_{1, +} \underline{\beta}_1 + \underline{\rmA}_{1, -}\overline{\beta}_1 = [-35.875],\\
        \overline{\rvh}_4 &= \overline{\rmA}_1 \rmX_0 + \lVert \overline{\rmA}_1\rVert_1 \epsilon + \overline{\rvd}_1 = [24.28],\\
        \underline{\rvh}_4 &= \underline{\rmA}_1 \rmX_0 - \lVert \underline{\rmA}_1\rVert_1 \epsilon + \underline{\rvd}_1 = [-42].
    \end{aligned}
\end{equation*}

As we can see from this example, the bounds from the backward mode LiRPA are the tightest compared to those from forward mode LiRPA and IBP, even if we reuse the intermediate relaxation results from the forward mode LiRPA.

\subsection{Existing LiRPA implementations}
\label{apd:lirpa_impl}

We list and compare a few notable LiRPA implementations in Table~\ref{table:lirpa_comparison}.

\begin{table}[ht]
      \caption{Comparison between different implementations for perturbation analysis. (``FF'' = FeedForward network).
   }
     \label{table:lirpa_comparison}
\adjustbox{max width=.58\textwidth}{
 \begin{minipage}{\textwidth}
\begin{threeparttable}
\centering
  \begin{tabular}{c|ccccccccc}
  \toprule[1pt]
   Method & Based On & Mode & Structure & Activation & Perturbation & Differentiability & Automatic\tnote{a} & Efficiency & Tightness\\
   \hline
   DiffAI~\cite{mirman2018differentiable} & PyTorch & Backward, IBP & FF+ResNet & ReLU & $\ell_\infty$ & Yes & No & GPU & ++\\
   IBP~\citep{gowal2018effectiveness,mirman2018differentiable} & TensorFlow & IBP & General & General & $\ell_{\infty}$ & Yes & No & GPU & - \\
   ERAN~\citep{maurereran2018} & C++/CUDA\tnote{b}
   & Backward, IBP, others\tnote{c}
   & General & General & $\ell_p$+semantic & No & No & Partially GPU & ++ \\
   Convex-Adv~\citep{wong2018provable} & PyTorch & Backward & FF+ResNet & ReLU & $\ell_p$ & Yes & No & Multi-GPU & + \\
   Fast-Lin~\citep{weng2018towards} & Numpy & Backward & FF (MLP) & ReLU & $\ell_p$ & No & No & CPU & + \\
   CROWN~\citep{zhang2018efficient} & Numpy & Backward & FF (MLP) & General & $\ell_p$ & No & No & CPU & ++ \\
   CROWN-IBP~\citep{zhang2018efficient} & PyTorch & Backward, IBP & FF & General & $\ell_p$ & Yes & No & Multi-GPU & ++\\
   Ours & PyTorch & Backward, Forward, IBP & General & General & General\tnote{d}
   & Yes & Yes & Multi-GPU & ++ \\
   \bottomrule[1pt]
  \end{tabular}
  \begin{tablenotes}
  \item[a] ``Automatic'' is defined as an user can easily obtain bounds using existing model source code, without manual conversion or implementation. 
  \item[b] ERAN has a TensorFlow frontend to read TensorFlow models, but its backend is written in C++ and partially CUDA.
  \item[c] Other types of bounds like k-ReLU~\citep{singh2019beyond} are provided, but typically much less efficient than IBP or backward mode perturbation analysis.
  \item[d] User supplied perturbation specifications.
  \end{tablenotes}
  \end{threeparttable}
  \end{minipage}
}
\end{table}

\section{Proofs of the Theorems}
\label{apd:proofs}
\subsection{Proof of Theorem 1}
\label{apd:theorm1_proof}

In Theorem~\ref{theorem:backward}, we bound node $o$ with:
\begin{equation}
    \sum_{i\in\rmV} \underline{\rmA}_i h_i(\rmX) + \underline{\rvd} 
    \leq h_o(\rmX) 
    \leq 
    \sum_{i\in\rmV} \overline{\rmA}_i h_i(\rmX) + \overline{\rvd} \quad \forall\rmX\in\sS.
    \label{eq:apd_backward}
\end{equation}
Initially, this inequality holds true with  
\begin{equation}
\underline{\rmA}_o=\overline{\rmA}_o=\rmI,\enskip 
\underline{\rmA}_i=\overline{\rmA}_i=\vzero(i\neq o),\enskip
\underline{\rvd} =\overline{\rvd}=\vzero,
\end{equation}
because then
$$ \sum_{i\in\rmV} \underline{\rmA}_i h_i(\rmX) + \underline{\rvd} = \sum_{i\in\rmV} \overline{\rmA}_i h_i(\rmX) + \overline{\rvd} = h_o(\rmX)$$
meets  \eqref{eq:apd_backward}.

Without loss of generality, we assume that the nodes are numbered in topological order, i.e., for each node $i$ and its input node $j\in u(i)$, $i>j$ holds true, and we assume that there are $n'$ independent nodes.
Then, we have $o=n$, and all the independent nodes have the smallest numbers.
This can be achieved via a topological sort for any computational graph.
We can also ignore nodes that node $o$ does not depend on.
With these assumptions, 
we show a lemma:
\begin{lemma}
	In Algorithm~\ref{alg:backward}, every dependent node $i(i>n')$ will be visited once and only once. And when node $i$ is visited, all nodes that depend on node $i$ must have been visited. 
	\label{lemma:alg2}
\end{lemma}
\begin{proof}
First, node $o$ is added to the queue and will be visited, and since it has no successor node, it will not be added to the queue again during the BFS.
We assume that node $i\dots n$ will be visited once and only once, and this is initially true with $i=o=n$.
For $i-1>n'$, we show that node $(i-1)$ will also be visited once and only once.
When node $i\dots n$ have all been visited, the successor nodes of node $(i-1)$ have been visited and $d_{i-1}=0$, and node $(i-1)$ is a dependent node.
Therefore, node $(i-1)$ will be added to the queue and visited.
From the assumption on node $i\dots n$, all nodes that depend on the successor nodes of node $(i-1)$ have also been visited.
Nodes that depend on node $(i-1)$ consist of the successor nodes of node $(i-1)$ and nodes that depend on these successors, and thus they have all been visited.
Since node $i\dots n$ will not be visited more than once, node $(i-1)$ will not be added to the queue by its successor nodes more than once.
Therefore, node $(i-1)$ will also be visited once and only once.
Using mathematical induction, we can prove that the lemma holds true for all node $i(i>n')$.
\end{proof}

According to Lemma~\ref{lemma:alg2}, every dependent node $i$ is visited once and exactly once.
When node $i$ is visited, Algorithm~\ref{alg:backward} performs the following changes to attributes 
$\underline{\rvd},~\overline{\rvd},~\underline{\rmA}_i,~\overline{\rmA}_i$ and $\underline{\rmA}_j,\overline{\rmA}_j (\forall j\in u(i))$:
\begin{equation}
    \underline{\rmA}_{j}+\!\!=\underline{\newA}_j,\ \ \overline{\rmA}_{j}+\!\!=\overline{\newA}_j,\ \ d_j-\!\!=1\quad\forall j\in u(i),\label{eq:change_1}
\end{equation}
\begin{equation}
    \underline{\rvd}+\!\!=\underline{\newd},\ \ \overline{\rvd}+\!\!=\overline{\newd},\ \ \underline{\rmA}_i\!\leftarrow\!\vzero, \ \ \overline{\rmA}_i\!\leftarrow\!\vzero,\label{eq:change_2}
\end{equation}
where $\underline{\newA}_j, \overline{\newA}_j, \underline{\mDelta}_j, \overline{\mDelta}_j$ come from oracle function $F_i$ as shown in \eqref{eq:backward_local}, and 
\begin{align*}
	\sum_{j\in u(i)} \underline{\newA}_j h_j(\rmX) + \underline{\newd} 
	\leq \underline{\rmA}_ih_i(\rmX),
	\enskip
	\overline{\rmA}_ih_i(\rmX)\leq
	\sum_{j\in u(i)} \overline{\newA}_j h_{j}(\rmX) + \overline{\newd}.
\end{align*}

Thereby, with changes in \eqref{eq:change_1} and \eqref{eq:change_2}, the linear lower bound in  \eqref{eq:apd_backward} becomes
\begin{align}
    h_o(\rmX)
    &\geq  \sum_{k\in\rmV} \underline{\rmA}_k h_k(\rmX) + \underline{\rvd}\nonumber\\
    &= \sum_{k\in\rmV,k\neq i,k\notin u(i)} \underline{\rmA}_k h_k(\rmX) +  \sum_{j\in u(i)} \underline{\rmA}_j h_j(\rmX) + \underline{\rmA}_i h_i(\rmX) +  \underline{\rvd}\nonumber\\
    &\geq \sum_{k\in\rmV,k\neq i,k\notin u(i)} \underline{\rmA}_k h_k(\rmX) +  \sum_{j\in u(i)} \underline{\rmA}_j h_j(\rmX) + 
    \sum_{j\in u(i)} \underline{\newA}_j h_j(\rmX) + \underline{\newd}  +  \underline{\rvd}\nonumber\\
    &=     \sum_{k\in\rmV,k\neq i,k\notin u(i)} \underline{\rmA}_k h_k(\rmX) +  \sum_{j\in u(i)} (\underline{\rmA}_j+\underline{\newA}_j) h_j(\rmX)
     + (\underline{\newd}  +  \underline{\rvd}),
     \label{eq:update_lin_bound}
\end{align}
which remains a valid linear lower bound in the form of \eqref{eq:apd_backward}.
Similarly, this also holds true for the linear upper bound.
In this way, $\underline{\rmA}_i$ and $\overline{\rmA}_i$ are propagated to its input nodes and set to $\vzero$.
Thereby the term w.r.t. $h_i(\rmX)$ is eliminated in the linear bounds, as shown in \eqref{eq:update_lin_bound}.

At this time, all successor nodes of node $i$ have been visited and will not been visited again.
Therefore, $\underline{\rmA}_i$ and $\overline{\rmA}_i$ will keep to be $\vzero$ after node $i$ is visited.
Therefore, when Algorithm~\ref{alg:backward} terminates, $\underline{\rmA}_i, \overline{\rmA}_i$ of all dependent node $i$ will be $\vzero$,
and thereby we will obtain linear bounds of node $o$ w.r.t. all the independent nodes.

\subsection{Proof of Theorem 2}
\label{apd:theorem2_proof}

Theorem~\ref{theorem:dp} shows that linear bounds under perturbation defined by synonym-based word substitution can be concretized with a dynamic programming.
Specifically, to concretize a linear lower bound, we need to compute
\begin{equation}
\begin{aligned}
    \underline{\rvh}_o &= \min_{\hat{w}_1,\hat{w}_2,\dots,\hat{w}_n}\underline{\rvb}_o+\sum_{t=1}^n \underline{\Tilde{\rmW}}_t e(\hat{w}_t)\quad
    \text{s.t.} &\sum_{t = 1}^n I(\hat{w}_t \neq w_t) \leq \delta,
\label{eq:apd_dp_goal}
\end{aligned}
\end{equation}
where $e(\hat{w}_t)$ is embedding of the $t$-th word in the input, 
$\underline{\Tilde{\rmW}}_t$ are columns in $\underline{\rmW}_o$ corresponding to the coefficients of $e(\hat{w}_t)$ in the linear bound.
In the dynamic programming, we compute 
$\underline{\rvg}_{i,j}(j\leq i)$ that denotes the lower bound of 
$
\underline{\rvb}_o + \sum_{t=1}^i \underline{\Tilde{\rmW}}_t e(\hat{w}_t)
$
when $j$ words among the first $i$ words $\hat{w}_1,  \dots, \hat{w}_i$ have been replaced.
If $\hat{w}_k$ has not been replaced, $\hat{w}_k=w_k$, otherwise $\hat{w}_k\in \sS(w_k)$.

For $i=0$, obviously $\underline{\rvg}_{0,0}=\underline{\rvb}_o$.
For $j=0$, $\hat{w}_1, \hat{w}_2, \cdots, \hat{w}_i$ must have not been replaced and thus $\hat{w}_t=w_t(1\leq t\leq i)$ holds true.
Therefore, 
$ \underline{\rvg}_{i,0} = \underline{\rvb}_o + \sum_{t=1}^i \underline{\Tilde{\rmW}}_t e(w_t).
$
For $i,j>0$, we consider whether $\hat{w}_i$ has been replaced. 
If $\hat{w}_i$ has not been replaced,
$\underline{\Tilde{\rmW}}_i e(\hat{w}_i)=\underline{\Tilde{\rmW}}_i e(w_i)$,
and $j$ words have been replaced among the first $i-1$ words. 
In this case, 
$ \underline{\rvb}_o + \sum_{t=1}^i \underline{\Tilde{\rmW}}_t e(\hat{w}_t) = \underline{\rvb}_o + \sum_{t=1}^{i-1} \underline{\Tilde{\rmW}}_t e(\hat{w}_t) + \underline{\Tilde{\rmW}}_i e(w_i)\geq \underline{\rvg}_{i-1,j}+\underline{\Tilde{\rmW}}_i e(w_i) $.
For the other case if $\hat{w}_i$ has been replaced, $j-1$ words have been replaced among the first $i-1$ words, and  
$ \underline{\rvb}_o + \sum_{t=1}^i \underline{\Tilde{\rmW}}_t e(\hat{w}_t) \geq \underline{\rvg}_{i-1,j-1} + \min_{w'}\{ \underline{\Tilde{\rmW}}_i e(w') \} $, where $w'\in\sS(w_i)$.
We combine these two cases and take the minimum of their results, and thus:
\begin{equation*}
    \underline{\rvg}_{i,j} =
    	\min(
    		\underline{\rvg}_{i-1,j} + \underline{\Tilde{\rmW}}_i e(w_i),\ \ 
    \underline{\rvg}_{i-1,j-1} \!+\! \min\nolimits_{w'}\{\underline{\Tilde{\rmW}}_i e(w') \} 
    	) \ \  (i,j>0)\quad
    	\text{s.t.} \ \  w' \in \sS(w_i).
\end{equation*}
The result of \eqref{eq:apd_dp_goal} is  $\min_{j=0}^\delta\underline{\rvg}_{n,j}$.
The upper bounds can also be computed in a similar way simply by changing from taking the minimum to taking the maximum in the above derivation.

\subsection{Proof of Theorem 3}
\label{apd:theorem3_proof}

In Theorem~\ref{theorem:loss_fusion}, we show that given concrete lower and upper bounds of $g_\theta(\rmX,y)$ as $\underline{g}_\theta(\rmX,y)$  and $\overline{g}_\theta(\rmX,y)$, with $S(\rmX,y)\!=\!\sum_{i\leq K} \exp(-[g_\theta(\rmX,y)]_i)$, we have 
\begin{equation}
{
\max_{\rmX\in\sS} L(f_\theta(\rmX), y)
\leq \log\overline{S}(\rmX,y)
\leq L(-\underline{g}_\theta(\rmX,y),y),
}
\label{apd:theorem3}
\end{equation}
where $\overline{S}(\rmX,y)$ is the upper bound of $S(\rmX,y)$ from the backward mode LiRPA.

$L(f_\theta(\rmX),y)$ is the cross entropy loss with softmax normalization, and 
\begin{align*}
L(f_\theta(\rmX), y) 
&= -\log \frac{[\exp(f_\theta(\rmX))]_y}{\sum_{i\leq K} [\exp(f_\theta(\rmX))]_i}\\
&= \log \sum_{i\leq K} \exp( [f_\theta(\rmX)]_i- [f_\theta(\rmX)]_y )\\
&= \log \sum_{i\leq K} \exp (-[g_\theta(\rmX,y)]_i )\\
&= \log S(\rmX, y). 
\end{align*}
Since $\log$ is a monotonic function, 
$$\max_{\rmX\in\sS} L(f_\theta(\rmX), y) = \log \max_{\rmX\in\sS} S(\rmX,y)\leq \log \overline{S}(\rmX,y).$$
And $L(-\underline{g}_\theta(\rmX,y),y)$ is an upper bound of $\max_{\rmX\in\sS} L(f_\theta(\rmX), y)$, since 
\begin{align*}
\max_{\rmX\in\sS} L(f_\theta(\rmX), y)
&\leq 
 \log \sum_{i\leq K} \exp (-\min_{\rmX\in\sS}[g_\theta(\rmX,y)]_i )\\
 &\leq  \log \sum_{i\leq K} \exp (-[\underline{g}_\theta(\rmX,y)]_i )\\
&=L(-\underline{g}_\theta(\rmX,y),y). 
\end{align*}

\begin{figure}[htb]  
\centering
\includegraphics[width=.6\textwidth]{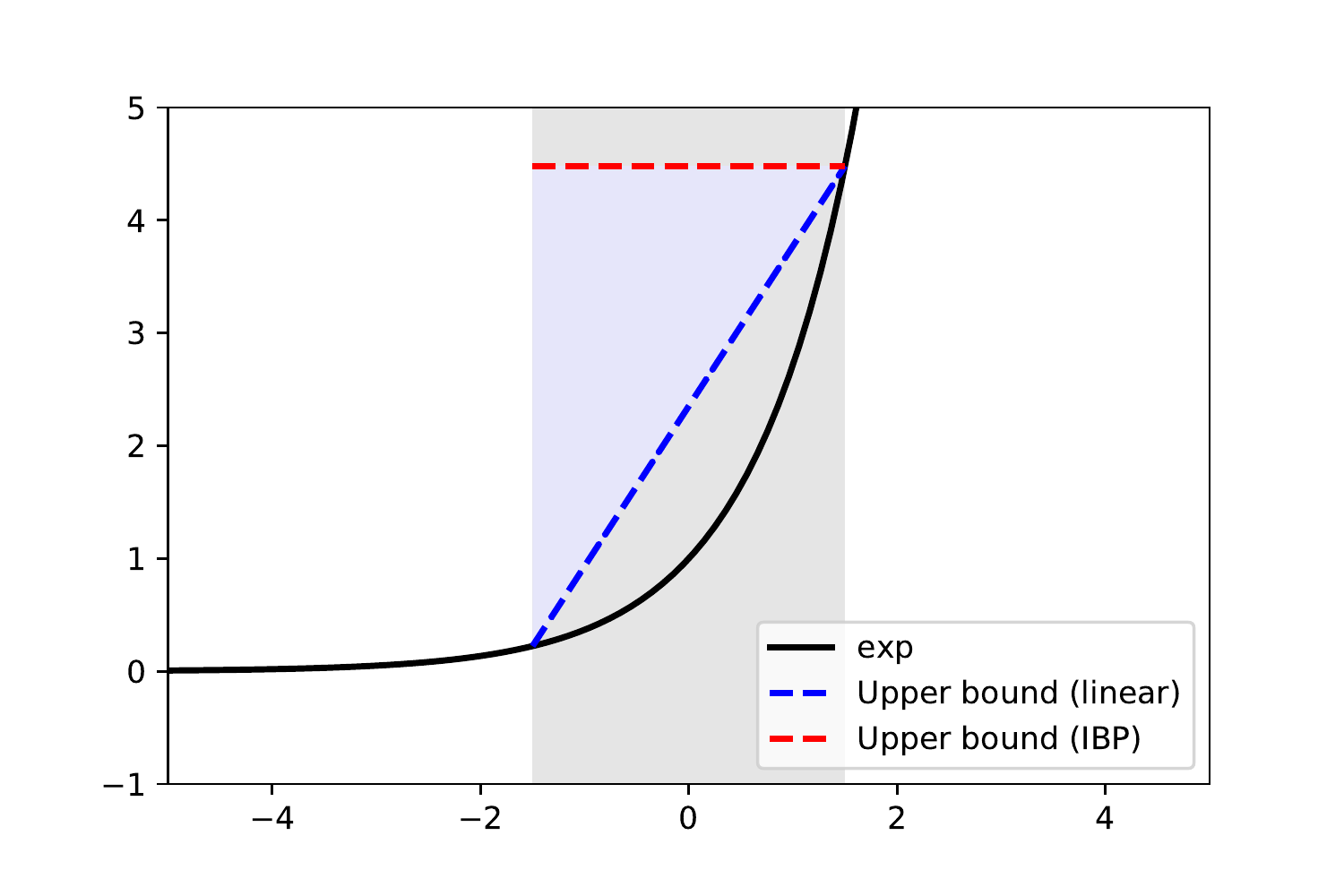} 
\caption {Illustration of different upper bounds of $\exp(x)$ within $x\in[-1.5,1.5]$. The linear bound (blue line) is a tighter bound than the IBP bound (red line).
The blue area stands for the gap between the two upper bounds. Note that for this particular setting of upper bounding $\overline{S}(\rmX, y)$ we need only upper bounds for this non-linear function.}
\label{fig:exp}
\end{figure}

Now we are going to show that $\log \overline{S}(\rmX, y) \leq L(-\underline{g}_\theta(\rmX, y), y)$. Here we assume that the concrete bounds of intermediate layers used for linear relaxations and also the concrete lower and upper bounds of $g_\theta(\rmX,y)$ (denoted as $\underline{g}_\theta(\rmX,y)$ and $\overline{g}_\theta(\rmX,y)$) are the same.


Computing $\sum_{i\leq K} \exp (-[\underline{g}_\theta(\rmX,y)]_i )$ is essentially propagating $\underline{g}_\theta(\rmX,y)$ through $\exp$ and summation in the loss function using IBP, while $\overline{S}(\rmX,y)$ is directly computed from the LiRPA bound of $S(\rmX,y)$.
Using $\Tilde{\rmA}$, a matrix of ones with size $1\times K$, to replace the summation, we can unify these two processes as computing the upper bound of 
$\Tilde{\rmA} \exp(-g_\theta(\rmX,y))$ using LiRPA with different relaxations for $\exp$.
For $\overline{S}(\rmX,y)$, the linear upper bound of $\exp(x)(l\leq x\leq u)$ is a line passing $(l,e^l)$ and $(u,e^u)$, while it is $e^u$ when computing $\sum_{i\leq K} \exp (-[\underline{g}_\theta(\rmX,y)]_i )$.
We illustrate the two different relaxations in Figure~\ref{fig:exp}.
Since elements in $\Tilde{\rmA}$ are all positive, the lower bound of $\exp(x)$ will not be involved, and thus with the same concrete bounds of $g_\theta$ the relaxation on $\exp$ in $\overline{S}(\rmX,y)$ is strictly tighter when $l<u$.

After relaxing $\exp$, we can obtain two linear upper bounds $\hat{\rmA} g_\theta(\rmX,y)+\hat{\rvd}$ from the two methods respectively,
where $\hat{\rmA}$ and $\hat{\rvd}$ are obtained by merging the relaxation of $\exp$ and $\Tilde{\rmA}$. 
Note that since the relaxed function $\exp(x) \leq e^u$ in IBP has no linear term, in this case $\hat{\rmA}=\vzero$ and the upper bound will simply be $\hat{\rvd}$. 
We then back propagate $\hat{\rmA} g_\theta(\rmX,y)+\hat{\rvd}$ to the input and concretize the bounds to get $\overline{S}(\rmX,y)$ and $\sum_{i\leq K} \exp (-[\underline{g}_\theta(\rmX,y)]_i )$ respectively.
Since in the calculation of linear bounds, the $\exp$ relaxation is the only difference and the relaxation for $\overline{S}(\rmX,y)$ is no looser than that for $\sum_{i\leq K} \exp (-[\underline{g}_\theta(\rmX,y)]_i )$, the upper linear bound of $\overline{S}(\rmX,y)$ is tighter than that of $\sum_{i\leq K} \exp (-[\underline{g}_\theta(\rmX,y)]_i )$, and we can conclude that for the concrete bounds $\overline{S}(\rmX,y)\leq\sum_{i\leq K} \exp (-[\underline{g}_\theta(\rmX,y)]_i )$ holds true, and thereby 
$\log\overline{S}(\rmX,y)
\leq L(-\underline{g}_\theta(\rmX,y),y)$.

\begin{remark}
Despite the assumptions involved above, in the implementation, we generally have different concrete bounds $\underline{g}_\theta(\rmX,y)$ and $\overline{g}_\theta(\rmX,y)$ for computing  $\overline{S}(\rmX, y)$ with loss fusion (e.g., our IBP+backward scheme), compared to the case of computing $L(-\underline{g}_\theta(\rmX,y))$ without loss fusion (e.g., the scheme used in CROWN-IBP~\citep{zhang2019towards}).
In the former case, $\underline{g}_\theta(\rmX,y)$ and $\overline{g}_\theta(\rmX,y)$ are regarded as intermediate bounds and obtained with IBP, while in the later case, $\underline{g}_\theta(\rmX,y)$ is obtained with LiRPA and $\overline{g}_\theta(\rmX,y)$ is unused.
Therefore, the relaxation on $\exp$ when using loss fusion may not be strictly tighter than the IBP bound in computing $L(-\underline{g}_\theta(\rmX,y))$.
\end{remark}

\section{Additional Details on Experiments}
\label{apd:additional_experiments}

\subsection{Details on Large-Scale Certified Defense}
\label{sec:app_vision_training}

\paragraph{Training settings}
In order to perform fair comparable experiments, for all experiments on training large-scale vision models (Table~\ref{table:vision_error_rate} and~\ref{tab:DSImageNet_error_rate}), we use a same setting for LiRPA and IBP. Across all datasets, the networks were trained using the Adam~\cite{kingma2014adam} optimizer with an initial learning rate of $5\times10^{-4}$. Also, gradient clipping with a maximum $\ell_2$ norm of $8$ is applied. We gradually increase $\epsilon$ within a fixed epoch length (800 epochs for CIFAR-10, 400 epochs for Tiny-ImageNet and 80 epochs for Downscaled-ImageNet). We uniformly divide the epoch length with a factor $0.4$, and exponentially increase $\epsilon$ during the former interval and linearly increase $\epsilon$ during the latter interval, so that to avoid a sudden growth of $\epsilon$ at the beginning stage. Following~\cite{zhang2019towards}, for LiRPA training, a hyperparameter $\beta$ to balance LiRPA bounds and IBP bounds for the output layer is set and gradually decreases from 1 to 0 (1 for only using LiRPA bounds and 0 for only using IBP bounds), as per the same schedule of $\epsilon$, and the end $\epsilon$ for training is set to $10\%$ higher than the one in test. All models are trained on 4 Nvidia GTX 1080TI GPUs (44GB GPU memory in total). For different datasets, we further have settings below:
\begin{itemize}
    \item \textbf{CIFAR-10} $\epsilon=\frac{8}{255}$.
    We train 2,000 epochs with batch size 256 in total, the first 200 epochs are clean training, then we gradually increase $\epsilon$ per batch with a $\epsilon$ schedule length of 800, finally we conduct 1,100 epochs pure IBP training. We decay the learning rate by $10\times$ at the 1,400-th and 1,700-th epochs respectively. During training, we add random flips and crops for data augmentation, and normalize each image channel, using the channel statistics from the training set.
    \item \textbf{Tiny-ImageNet} $\epsilon=\frac{1}{255}$. We train 800 epochs with batch size 120 in total (for WideResNet, we reduce batch size to 110 due to limited GPU memory), the first 100 epochs are clean training, then we gradually increase $\epsilon$  per batch with a $\epsilon$ schedule length of 400, finally we conduct 500 epochs of pure IBP training. We decay the learning rate by $10\times$ at the 600-th and 700-th epochs respectively. During training, we use random crops of 56 $\times$ 56 and random flips. During testing, we use a central 56 $\times$ 56 crop. We also normalize each image channel, using the channel statistics from the training set.
    \item \textbf{Downscaled-ImageNet} $\epsilon=\frac{1}{255}$. We train 240 epochs with batch size 110 in total, the first 100 epochs are clean training, then we gradually increase $\epsilon$  per batch with a $\epsilon$ schedule length of 80, finally we conduct 60 epochs of pure IBP training. We decay the learning rate by $10\times$ at the 200-th and 220-th epochs respectively. During training, we use random crops of 56 $\times$ 56 and random flips. During testing, we use a central 56 $\times$ 56 crop. We also normalize each image channel, using the channel statistics from the training set.
\end{itemize}
All verified error numbers are evaluated on the test set using IBP with $\epsilon=\frac{8}{255}$ for CIFAR-10 and  $\epsilon=\frac{1}{255}$ for Tiny-ImageNet and Downscaled-ImageNet.
\paragraph{Model Structures}
The details of vision model structures we used  are described bellow (note that we omit the final linear layer which has 10 neurons for CIFAR-10 and 200 neurons for Tiny-ImageNet):
\begin{itemize}
    \item \textbf{CNN-7+BN} $5\times$ Conv-BN-ReLU layers with $\{64, 64, 128, 128, 128\}$ filters respectively, and a linear layer with $512$ neurons.
    \item \textbf{DenseNet} $\{2, 4, 4\}$ Dense blocks with growth rate 32 and  a linear layer with $512$ neurons.
    \item \textbf{WideResNet} $3\times$ Wide basic blocks ($6\times$ Conv-ReLU-BN layers) with widen factor = 4 for CIFAR-10,  widen factor = 10 for Tiny-ImageNet and Downscaled-ImageNet. An additional linear layer with $512$ neurons is added for CIFAR-10.
    \item \textbf{ResNeXt} $\{1, 1, 1\}$ blocks for CIFAR-10 and $\{2, 2, 2\}$ blocks for Tiny-ImageNet and cardinality = 2, bottleneck width = 32 and a linear layer with $512$ neurons.

\end{itemize}

It is worthwhile to mention that both~\cite{zhang2019towards} and~\cite{zhu2020improving} conducted experiments on expensive 32 TPU cores which has up to 512 GB TPU memory in total. In comparison, our framework with loss fusion can be quite efficient working on 44 GB GPU memory.

Moreover, the running time with maximum batch size on 4 Nvidia GTX 1080TI GPUs of all models on two datasets is reported in Table~\ref{table:cost_max_batch}. Note that large-scale models cannot be trained with previous LiRPA methods without loss fusion, even if the mini-batch size on each GPU is only 1 for DenseNet and WideResNet.

\begin{table}[htb]
\centering
\caption{Per-epoch training time and memory usage of the 4 large models on CIFAR-10 and Tiny-ImageNet with maximum batch size for 4  Nvidia GTX 1080TI GPUs. ``LF''=loss fusion. ``OOM''= out of memory. Numbers in parentheses are relative to natural training time.}
\footnotesize
  \adjustbox{max width=1.\textwidth}{
\begin{tabular}{c|c|cccc|cccc}
\toprule[1pt]
\multirow{2}{*}{Data}& \multirow{2}{*}{Training method}& \multicolumn{4}{c|}{Wall clock time (s)}  & \multicolumn{4}{c}{Maximum batch size}            \\
          \cline{3-10}
          && Natural & IBP & LiRPA w/o LF & LiRPA w/ LF &  Natural & IBP & LiRPA w/o LF & LiRPA w/ LF  \\
          \hline
\multirow{4}{*}{CIFAR-10}&CNN-7+BN & 7.59 & 11.17 (1.54$\times$)& 46.52 (6.13$\times$) & 28.20 (3.71$\times$) & 9500 & 3000 & 600  & 1700 \\
&DenseNet & 9.23 & 37.25 (4.04$\times$) & 187.45 (20.31$\times$) & 74.54 (8.08$\times$) & 2500 & 800  & 150 & 400 \\
&WideResNet& 12.08 & 37.70 (3.12$\times$) & 236.66 (19.59$\times$) & 65.72 (5.44$\times$) & 3000 & 1000 & 160  & 550\\
&ResNeXt & 6.83 & 19.70 (2.88$\times$)& 130.37 (19.09$\times$) & 43.65 (6.39$\times$) & 4000 & 1200  & 260 & 700 \\
\midrule
\multirow{4}{*}{Tiny-ImageNet}
&CNN-7+BN & 22.17 &  56.54 (2.55$\times$) & 4344.05 (195.94$\times$)  & 98.04 (4.42$\times$) & 3600 & 1100  & 12  & 600  \\
&DenseNet & 50.60 & 223.63 (4.42$\times$)  & OOM  &  474.66 (9.38$\times$)  & 800 & 240 & OOM & 120 \\
&WideResNet & 98.01 & 370.68 (3.78$\times$)  & OOM & 604.70 (6.17$\times$)  & 600  & 200  & OOM & 110 \\
&ResNeXt & 21.52 & 59.42 (2.76$\times$)  & 5580.52 (259.32$\times$)  & 119.34 (5.55$\times$)   & 3200  & 900  & 12  & 500 \\

\bottomrule[1pt]
\end{tabular}
}
\label{table:cost_max_batch}
\end{table}

\subsection{Details on Verifying and Training NLP Models}
\label{apd:SST_training}
For the perturbation specification defined on synonym-based word substitution, each word $w$ has a substitution set $\sS(w)$, such that the actual input word $w'\in \{w\}\cup\sS(w)$. 
We adopt the approach for constructing substitution sets used by \citet{jia2019certified}. 
For a word $w$ in a input sentence, they first follow \citet{alzantot2018generating} to find the nearest 8 neighbors of $w$ in a counter-fitted word embedding space where synonyms are generally close while antonyms are generally far apart. 
They then apply a language model to only retain substitution words that the log-likelihood of the sentence after word substitution does not decrease by more than 5.0, which is also similar to the approach by \citet{alzantot2018generating}. 
We reuse their open-source code\footnote{\url{https://bit.ly/2KVxIFN}} to pre-compute the substitution sets of words in all the examples.
Note that although we use the same approach for constructing the lists of substitution words as \cite{jia2019certified}, our perturbation space is still different from theirs, because we follow  \citet{huang2019achieving} and allow setting a small budget $\delta$ that limits the maximum number of words to be replaced simultaneously~\citep{ko2019popqorn,gao2018black}.
We do not adopt the synonym list from \citet{huang2019achieving} as it appears to be not publicly available when this work is done.

We use two models in the experiments for sentiment classification: Transformer and LSTM. For Transformer, we use a one-layer model, with 4 attention heads, a hidden size of 64, and ReLU activations for feed-forward layers. 
Following \citet{shi2020robustness}, we also remove the variance related terms in layer normalization, which can make Transformer easier to be verified while keeping comparable clean accuracies. 
For the LSTM, we use a one-layer bidirectional model, with a hidden size of 64. 
The vocabulary is built from the training data and includes all the words that appear for at least twice. 
Input tokens to the models are truncated to no longer than 32. 

In the certified defense, although we are not using $\ell_p$ norm perturbations,
we have an artifial
$\epsilon$ that manually shrinks the gap between the clean input and perturbed input during the warmup stage, which makes the objective easier to be optimized~\cite{gowal2018effectiveness,jia2019certified}.
Specifically, for clean input word $w_i$ and actual input word $\hat{w}_i$, 
we shrink the gap between the embeddings of $w_i$ and $\hat{w}_i$ respectively:
$$ e(\hat{w}_i)\leftarrow \eps e(\hat{w}_i) + (1-\eps) e(w_i).$$
 
$\eps$ is linearly increased from 0 to 1 during the first 10 warmup epochs.
We then train the model for 15 more epochs with $\eps=1$.
During the first 20 epochs, all the nodes on the parse trees of training examples are used, and later we only use the root nodes, i.e., the full text only.
The models are trained using Adam optimizer~\cite{kingma2014adam}, and the learning rate is set to $10^{-4}$ for Transformer and $10^{-3}$ for LSTM. We also use gradient clipping with a maximum norm of 10.0. 
When using LiRPA bounds for training, we combine bounds by LiRPA and IBP weighted by a coefficient $\beta(0\leq \beta\leq 1)$ and $(1-\beta)$ respectively, and $\beta$ decreases from 1 to 0 during the warmup stage, following CROWN-IBP~\citep{zhang2019towards} as also mentioned in Appendix~\ref{sec:app_vision_training}.
In this setting, since we use pure IBP for training in the last epochs, we actually end up training the models on $\delta=\infty$ since IBP for LSTM and Transformer does not consider $\delta$ (see the next paragraph).
But we still use LiRPA bounds with the given non-trivial $\delta$ for testing.
Alternatively, for \emph{IBP+Backward (alt.)} in the experiments, we always use LiRPA bounds and set $\beta=1$.
And for this setting, the models tend to have a lower verified accuracy when tested on a $\delta$ larger than that in the training, as shown in Sec.~\ref{sec:exp}.

\citet{huang2019achieving} has a convex hull method to handle word replacement with a budget limit $\delta$ in IBP.
For a word sequence $w_1, w_2, \cdots, w_l$, they  construct a convex hull for the input node $1$.
They consider the perturbation of each word $w_i$, and for each possible $\hat{w}_i\in \{ w_i \} \cup \sS(w_i)$, they add vector
$
	[e(w_{1\cdots i-1}); 
	e(w_i)+\delta(e(\hat{w}_i)-e(w_i));
	e(w_{i+1\cdots l})]
$
to the convex hull.
The convex hull is an over-estimation of $h_1(\rmX)$.
They require the first layer of the network to be an affine layer and concretize the convex hull to interval bounds after passing the first layer, where each vertex in the convex hull is passed through the first layer respectively and they then take the interval lower and upper bound of all the vertexes in the convex hull.
They worked on CNN, but on Transformer when there is no interaction between different sequence positions in the first layer, their method is a $(\delta-1)$-time more over-estimation than simply assuming all the words can be replaced at the same time, and this method cannot work either when the first layer is not an affine layer.
Therefore, for verifying and training LSTM and Transformer with IBP, we can only adopt the baseline in \citet{jia2019certified} without considering $\delta$.
In contrast, our dynamic programming method for concretizing linear bounds under the synonym-based word substitution scenario in Sec.~\ref{sec:spec} takes the budget into consideration regardless of the network structure.

\subsection{Details on Training for a Flat Objective}
\label{sec:app_detail_flat}
\paragraph{Hyperparameter Setting}
\label{apd:weight_perturbation}
For training the three-layer MLP model we used in weight perturbation experiments, we follow similar training strategy in vision models. The differences are summarized here: We use the SGD optimizer with an initial learning rate of $0.1$ and decay the learning rate with a factor of $0.5$ after $\epsilon$ increases. We use $\ell_2$ norm with $\epsilon = 0.1$ to bound the weights of all three layers and linearly increase $\epsilon$ per batch.

\paragraph{Certified Flatness} Using bounds obtained from LiRPA, we can obtain a certified upper bound on training loss. We define the flatness based on certified training cross entropy loss at a point $\theta^* = [\rvw^*_1, \rvw^*_2, \cdots, \rvw^*_K]$ as:
\begin{equation}
    \mathcal{F} = \mathcal{L}(-\underline{\rvh}(\rvx, \theta^*, \bm{\epsilon});\, y) - \mathcal{L}(\rvh(\rvx, \theta^*);\, y) \geq \max_{\rvw \in \sS} \mathcal{L}(\theta) - \mathcal{L}(\theta^*).
    \label{eq:flatness}
\end{equation}
A small $\mathcal{F}$ guarantees that $\mathcal{L}$ does not change wildly around $\theta^*$. Note that since the weight of each layer can be in quite different scales, we use a normalized $\overline{\epsilon} = 0.01$ and set $\epsilon_i = \|\rvw_i\|_2 \overline{\epsilon}$. This also allows us to make fair comparisons between models with weights in different scales. The flatness $\mathcal{F}$ of the models we obtained are shown in Table\,\ref{table: flatness}. As we can see, the models trained by ``flat'' objective show extraordinary smaller flatness $\mathcal{F}$ compare with the nature trained models on bot MNIST and FashionMNIST with all combination of dataset sizes and batch sizes. The results also fit the observation of training loss landscape in Figure~\ref{fig: optimization_loss}.

\begin{table}[htb]\small
 \centering
  \caption{The flatness $\mathcal{F}$ of naturally trained models and models trained using the ``flat'' objective~\eqref{eq:flatness} with different dataset sizes (10\%, 1\%) and batch sizes ($0.01N$, $0.1N$, $N$). A small $\mathcal{F}$ guarantees that $\mathcal{L}$ does not change wildly around $\theta^*$ (model parameters found by SGD). The flat objective provably reduces the range of objective around $\theta^*$.
  } 
    \adjustbox{max width=.5\textwidth}{
\begin{tabular}{c|ccc|ccc}
\toprule[1pt]
  & \multicolumn{6}{c}{ MNIST  }  \\
\cline{2-7}
 & \multicolumn{3}{c|}{ nature training  }  & \multicolumn{3}{c}{ ``flat'' objective  }\\
\cline{2-7}
            & $0.01N$ & $0.1N$ & $N$  & $0.01N$ & $0.1N$ & $N$ \\
\midrule[1pt]
$10\%$       &  2.79  & 3.45  & 4.55 & 0.97  & 1.12  & 1.83  \\
 \midrule
$1\%$       & 2.96 & 3.85 & 4.77 & 1.10 & 0.95 & 1.44 \\
 \midrule
& \multicolumn{6}{c}{ FashionMNIST  }\\

 \midrule[1pt]
$10\%$       & 7.89    & 7.95   & 9.60 & 2.49  & 1.81  & 1.94    \\
 \midrule
$1\%$     & 7.86    & 6.43  & 9.55 & 2.52 & 1.79 & 1.98  \\

 \bottomrule[1pt]
\end{tabular}
}
  \label{table: flatness}
\end{table}